\documentclass[oneside,11pt]{article} 
\renewenvironment{abstract}
  {{\centering\large\bfseries Abstract\par}\vspace{0.7ex}%
    \bgroup
       \leftskip 20pt\rightskip 20pt\small\noindent\ignorespaces}%
  {\par\egroup\vskip 0.25ex}

\newenvironment{keywords}
{\vspace{0.05in}\bgroup\leftskip 20pt\rightskip 20pt \small\noindent{\bfseries
Keywords:} \ignorespaces}%
{\par\egroup\vskip 0.25ex}

\usepackage{times}
\usepackage{graphicx}
\usepackage{algorithm,algorithmic}
\usepackage{amsfonts,amssymb,amsmath,amsthm} 
\usepackage{xcolor}
\usepackage{url}
\usepackage{booktabs}
\usepackage{enumitem}
\usepackage[colorlinks,
            linkcolor=blue,
            citecolor=blue,
            urlcolor=magenta,
            linktocpage,
            plainpages=false]{hyperref}

\theoremstyle{plain}
\newtheorem{theorem}{Theorem}
\newtheorem{lemma}[theorem]{Lemma}
\newtheorem{proposition}[theorem]{Proposition}
\newtheorem{corollary}[theorem]{Corollary}
\theoremstyle{definition}
\newtheorem{definition}[theorem]{Definition}
\newtheorem{remark}[theorem]{Remark}

\numberwithin{equation}{section}

\newcommand{\defeq}{:=}

\renewcommand{\(}{\left(}
\renewcommand{\)}{\right)}

\DeclareMathOperator*{\argmin}{arg\,min}

\DeclareMathOperator*{\E}{\mathbb{E}}
\newcommand{\st}{\mathrm{s.t.}}

\newcommand{\bA}{\boldsymbol{A}}

\newcommand{\bM}{\boldsymbol{M}}

\newcommand{\bI}{\boldsymbol{I}}

\newcommand{\bSigma}{\boldsymbol{\Sigma}}

\newcommand{\bx}{\boldsymbol{x}}
\newcommand{\bw}{\boldsymbol{w}}
\newcommand{\by}{\boldsymbol{y}}

\newcommand{\be}{\boldsymbol{e}}
\newcommand{\br}{\boldsymbol{r}}
\newcommand{\bu}{\boldsymbol{u}}
\newcommand{\bv}{\boldsymbol{v}}
\newcommand{\ba}{\boldsymbol{a}}
\newcommand{\bb}{\boldsymbol{b}}

\newcommand{\R}{\mathbb{R}}
\newcommand{\Rn}{\mathbb{R}^n}

\newcommand{\Rd}{\mathbb{R}^d}

\newcommand{\Rnd}{\mathbb{R}^{n\times d}}

\newcommand{\zeronorm}[1]{\left\lVert #1 \right\rVert_{0}}
\newcommand{\twonorm}[1]{\left\lVert #1 \right\rVert_{2}}

\newcommand{\abs}[1]{\left\lvert #1 \right\rvert}

\newcommand{\fractwo}[1]{\frac{#1}{2}}
\newcommand{\supp}[1]{\mathrm{supp}\(#1\)}

\newcommand{\trans}{^{\top}}

\newcommand{\inner}[2]{\left\langle #1, #2 \right\rangle}

\newcommand{\const}{\mathrm{{C}}}

\newcommand{\PO}[1]{\mathcal{P}_{\Omega}\(#1\)}
\newcommand{\Hk}[1]{\mathcal{H}_{k}\(#1\)}
\newcommand{\order}[1]{\mathcal{O}\( #1 \)}

\newcommand{\xtilde}{\widetilde{\bx}}
\newcommand{\xhat}{\widehat{\bx}}
\newcommand{\xopt}{\bx_{\mathrm{opt}}}
\newcommand{\mutilde}{\widetilde{\boldsymbol{\mu}}}
\newcommand{\bphi}{\boldsymbol{\phi}}
\newcommand{\bPhi}{\boldsymbol{\Phi}}
\newcommand{\beps}{\boldsymbol{\varepsilon}}
\newcommand{\bzero}{\boldsymbol{0}}

\graphicspath{{fig/}}

\title{A Tight Bound of Hard Thresholding}
\author{
{\bf Jie Shen}\\
Rutgers University\\
Piscataway, NJ 08854, USA\\
\texttt{js2007@rutgers.edu}
\and
{\bf Ping Li}\\
Baidu Research \\
Bellevue, WA 98004, USA\\
\texttt{pingli98@gmail.com}
}
\date{}
\begin{document}
\maketitle

\begin{abstract}
This paper is concerned with the hard thresholding operator which sets all but the $k$ largest absolute elements of a vector to zero. We establish a {\em tight} bound to quantitatively characterize the deviation of the thresholded solution from a given signal. Our theoretical result is universal in the sense that it holds for all choices of parameters, and the underlying analysis depends only on fundamental arguments in mathematical optimization. We discuss the implications for two domains:

\vspace{0.05in}
\noindent {\bf Compressed Sensing.} \
On account of the crucial estimate, we bridge the connection between the restricted isometry property (RIP) and the sparsity parameter for a vast volume of hard thresholding based algorithms, which renders an improvement on the RIP condition especially when the true sparsity is unknown. This suggests that in essence, many more kinds of sensing matrices or fewer measurements are admissible for the data acquisition procedure.

\vspace{0.05in}
\noindent {\bf Machine Learning.} \
In terms of large-scale machine learning, a significant yet challenging problem is learning accurate sparse models in an efficient manner. In stark contrast to prior work that attempted the $\ell_1$-relaxation for promoting sparsity, we present a novel stochastic algorithm which performs hard thresholding in each iteration, hence ensuring such parsimonious solutions. Equipped with the developed bound, we prove the {\em global linear convergence} for a number of prevalent statistical models under mild assumptions, even though the problem turns out to be non-convex.

\end{abstract}


\section{Introduction}\label{sec:intro}
Over the last two decades, pursuing sparse representations has emerged as a fundamental technique throughout bioinformatics~\cite{olshausen1997sparse}, statistics~\cite{tibshirani1996regression,efron2004least}, signal processing~\cite{chen1998atomic,donoho2006stable,donoho2006compressed,candes2008introduction} and mathematical science~\cite{chandrasekaran2012convex}, to name just a few. In order to obtain a sparse solution, a plethora of practical algorithms have been presented, among which two prominent examples are greedy pursuit and convex relaxation~\cite{tropp2010computational}. For instance, as one of the earliest greedy algorithms, orthogonal matching pursuit~(OMP)~\cite{pati1993orthogonal} repeatedly picks a coordinate as the potential support of a solution. While OMP may fail for some deterministic sensing matrices, \cite{tropp2004greed,tropp2007signal} showed that it recovers the true signal with high probability when using random matrices such as Gaussian. Inspired by the success of OMP, the two concurrent work of compressive sampling matching pursuit~(CoSaMP)~\cite{needell2009cosamp} and subspace pursuit~(SP)~\cite{dai2009subspace} made improvement by selecting multiple coordinates  followed by a pruning step in each iteration, and the recovery condition was framed under the restricted isometry property~(RIP)~\cite{candes2005decoding}. Interestingly, the more careful selection strategy of CoSaMP and SP leads to an optimal sample complexity. The iterative hard thresholding~(IHT) algorithm~\cite{daubechies2004iterative,blumensath2008iterative,blumensath2009iterative} gradually refines the iterates by gradient descent along with truncation. \cite{foucart2011hard} then developed a concise algorithm termed hard thresholding pursuit~(HTP), which combined the idea of CoSaMP and IHT, and showed that HTP is superior to both in terms of the RIP condition. \cite{jain2011orthogonal} proposed an interesting variant of the HTP algorithm and obtained a sharper RIP result. Recently,~\cite{bahmani2013greedy} and~\cite{yuan2018gradient} respectively extended CoSaMP and HTP to general objective functions, for which a global convergence was established.

Since the sparsity constraint counts the number of non-zero components which renders the problem non-convex, the $\ell_1$-norm was suggested as a convex relaxation dating back to basis pursuit~\cite{chen1998atomic,donoho2008fast} and Lasso~\cite{tibshirani1996regression}. The difference is that Lasso looks for an $\ell_1$-norm constrained solution that minimizes the residual while the principle of basis pursuit is to find a signal with minimal $\ell_1$-norm that fits the observation data. \cite{candes2005decoding} carried out a detailed analysis on the recovery performance of basis pursuit. Another popular estimator in the high-dimensional statistics is the Dantzig selector~\cite{candes2007dantzig} which, instead of constraining the residual of the linear model, penalizes the maximum magnitude of the gradient. From a computational perspective, both basis pursuit and Dantzig selector can be solved by linear programming, while Lasso is formulated as a quadratic problem. Interestingly, under the RIP condition or the uniform uncertainty assumption~\cite{candes2006robust}, a series of work showed that exact recovery by convex programs is possible as soon as the observation noise vanishes~\cite{candes2005decoding,candes2008restricted,wainwright2009sharp,cai2010new,foucart2012sparse}.

In this paper, we are interested in the hard thresholding~(HT) operator underlying a large body of the developed algorithms in compressed sensing~(e.g., IHT, CoSaMP, SP), machine learning~\cite{yuan2013truncated}, and statistics~\cite{ma2013sparse}. Our motivation is two-fold. From a high level, compared to the convex programs, these HT-based algorithms are always orders of magnitude computationally more efficient, hence more practical for large-scale problems~\cite{tropp2010computational}. Nevertheless, they usually require a more stringent condition to guarantee the success. This naturally raises an interesting question of whether we can derive milder conditions for HT-based algorithms to achieve the best of the two worlds. For practitioners, to address the huge volume of data, a popular strategy in machine learning is to appeal to stochastic algorithms that sequentially update the solution. However, as many researchers observed~\cite{langford2009sparse,duchi2009fobos,xiao2010dual}, it is hard for the $\ell_1$-based stochastic algorithms to preserve the sparse structure of the solution as the batch solvers do. This immediately poses the question of whether we are able to apply the principal idea of hard thresholding to stochastic algorithms while still ensuring a fast convergence.

To elaborate the problem more precisely, let us first turn to some basic properties of hard thresholding along with simple yet illustrative cases. For a general vector $\bb \in \Rd$, the hard thresholded signal $\Hk{\bb}$ is formed by setting all but the largest~(in magnitude) $k$ elements of $\bb$ to zero. Ties are broken lexicographically. Hence, the hard thresholded signal $\Hk{\bb}$ is always $k$-{sparse}, i.e., the number of non-zero components does not exceed $k$. Moreover, the resultant signal $\Hk{\bb}$ is a best $k$-sparse approximation to $\bb$ in terms of any $\ell_p$ norm~($p \geq 1$). That is, for any $k$-sparse vector $\bx$
\begin{equation*}
\lVert{\Hk{\bb} - \bb}\rVert_p \leq \lVert{\bx - \bb}\rVert_p.
\end{equation*}
In view of the above inequality, a broadly used bound in the literature for the deviation of the thresholded signal is as follows:
\begin{equation}\label{eq:old}
\twonorm{\Hk{\bb} - \bx} \leq 2\twonorm{\bb - \bx}.
\end{equation}
To gain intuition on the utility of~\eqref{eq:old} and to spell out the importance of offering a tight bound for it, let us consider the compressed sensing problem as an example for which we aim to recover the true sparse signal $\bx$ from its linear measurements. Here, $\bb$ is a good but dense approximation to $\bx$ obtained by, e.g., full gradient descent. Then~\eqref{eq:old} justifies that in order to obtain a structured (i.e., sparse) approximation by hard thresholding, the distance of the iterate to the true signal $\bx$ is upper bounded by a multiple of $2$ to the one before.  For comparison, it is worth mentioning that $\ell_1$-based convex algorithms usually utilize the soft thresholding operator which enjoys the non-expansiveness property~\cite{saga}, i.e., the iterate becomes closer to the optimum after projection. This salient feature might partially attribute to the wide range of applications of the $\ell_1$-regularized formulations. Hence, to derive comparable performance guarantee, tightening the bound~\eqref{eq:old} is crucial in that it controls how much deviation the hard thresholding operator induces. This turns out to be more demanding for {stochastic gradient methods}, where the proxy $\bb$ itself is affected by the randomness of sample realization. In other words, since $\bb$ does not minimize the objective function (it only optimizes the objective in expectation), the deviation~\eqref{eq:old} makes it more challenging to analyze the convergence behavior. As an example,~\cite{nguyen2014linear} proposed a stochastic solver for general sparsity-constrained programs but suffered a non-vanishing optimization error due to randomness. This indicates that to mitigate the randomness barrier, we have to seek a better bound to control the precision of the thresholded solution and the variance.

\subsection{Summary of Contributions}
In this work, we make three contributions:
\begin{enumerate}
\item We examine the tightness of~\eqref{eq:old} that has been used for a decade in the literature and show that the equality therein will never be attained. We then improve this bound and quantitatively characterize that the deviation is inversely proportional to the value of $\sqrt{k}$. Our bound is tight, in the sense that the equality we build can be attained for specific signals, hence cannot be improved if no additional information is available. Our bound is universal in the sense that it holds for all choices of $k$-sparse signals $\bx$ and for general signals $\bb$.

\item Owing to the tight estimate, we demonstrate how the RIP~(or RIP-like) condition assumed by a wide range of hard thresholding based algorithms can be relaxed. In the context of compressed sensing, it means that in essence, many more kinds of sensing matrices or fewer measurements can be utilized for data acquisition. For machine learning, it suggests that existing algorithms are capable of handling more difficult statistical models.

\item Finally, we present an computationally efficient algorithm that applies hard thresholding in large-scale setting and we prove its linear convergence to a global optimum up to the statistical precision of the problem. We also prove that with sufficient samples, our algorithm identifies the true parameter for prevalent statistical models. Returning to~\eqref{eq:old}, our analysis shows that only when the deviation is controlled below the multiple of $1.15$ can such an algorithm succeed. This immediately implies that the conventional bound~\eqref{eq:old} is not applicable in the challenging scenario.
\end{enumerate}

\subsection{Notation}
Before delivering the algorithm and main theoretical results, let us instate several pieces of notation that are involved throughout the paper. We use bold lowercase letters, e.g., $\bv$, to denote a vector (either column or row) and its $i$th element is denoted by $v_i$. The $\ell_2$-norm of a vector $\bv$ is denoted by $\twonorm{\bv}$. The support set of $\bv$, i.e., indices of non-zeros, is denoted by $\supp{\bv}$ whose cardinality is written as $\abs{\supp{\bv}}$ or $\zeronorm{\bv}$. We write bold capital letters such as $\bM$ for matrices and its $(i, j)$-th entry is denoted by $m_{ij}$. The capital upright letter $\const$ and its subscript variants~(e.g., $\const_0, \const_1$) are reserved for absolute constants whose values may change from appearance to appearance.

For an integer $d > 0$, suppose that $\Omega$ is a subset of $\{1,\ 2,\ \dots,\ d\}$. Then for a general vector $\bv \in \Rd$, we define $\PO{\cdot}$ as the orthogonal projection onto the support set $\Omega$ which retains elements contained in $\Omega$ and sets others to zero. That is,
\begin{align*}
\(\PO{\bv} \)_i =\begin{cases}
v_i,\ &\text{if}\ i \in \Omega,\\
0,\ &\text{otherwise}.
\end{cases}
\end{align*}
In particular, let $\Gamma$ be the support set indexing the $k$ largest absolute components of $\bv$. In this way, the hard thresholding operator is given by
\begin{equation*}
\Hk{\bv} = \mathcal{P}_{\Gamma}(\bv).
\end{equation*}
We will also use the orthogonal projection of a vector $\bv$ onto an $\ell_2$-ball with radius $\omega$. That is,
\begin{equation*}
\Pi_{\omega}(\bv)=\frac{ \bv}{ \max\{ 1, \twonorm{\bv} / \omega \}}.
\end{equation*}

\subsection{Roadmap}
We present the key tight bound for hard thresholding in Section~\ref{sec:key}, along with a justification why the conventional bound~\eqref{eq:old} is not tight. We then discuss the implications of the developed tight bound to compressed sensing and machine learning in Section~\ref{sec:imp}, which shows that the RIP or RIP-like condition can be improved for a number of popular algorithms. Thanks to our new estimation, Section~\ref{sec:alg} develops a novel stochastic algorithm which applies hard thresholding to large-scale problems and establishes the global linear convergence. A comprehensive empirical study on the tasks of sparse recovery and binary classification is carried out in Section~\ref{sec:exp}. Finally, We conclude the paper in Section~\ref{sec:con} and all the proofs are deferred to the appendix.

\section{The Key Bound}\label{sec:key}
We argue that the conventional bound~\eqref{eq:old} is not tight, in the sense that the equality therein can hardly be attained. To see this, recall how the bound was derived for a $k$-sparse signal $\bx$ and a general one $\bb$:
\begin{equation*}
\twonorm{\Hk{\bb} - \bx} = \twonorm{ \Hk{\bb} - \bb + \bb - \bx} \stackrel{\xi}{\leq} \twonorm{\Hk{\bb} - \bb} + \twonorm{\bb - \bx} \leq 2 \twonorm{\bb - \bx},
\end{equation*}
where the last inequality holds because $\Hk{\bb}$ is a best $k$-sparse approximation to $\bb$. The major issue occurs in $\xi$. Though it is the well-known triangle inequality and the equality could be attained if there is no restriction on the signals $\bx$ and $\bb$, we remind here that the signal $\bx$ does have a specific structure~--~it is $k$-sparse. Note that in order to fulfill the equality in $\xi$, we must have $\Hk{\bb} - \bb = \gamma(\bb - \bx)$ for some $\gamma \geq 0$, that is,
\begin{equation}\label{eq:key1}
\Hk{\bb} = (\gamma+1)\bb - \gamma \bx.
\end{equation}
One may verify that the above equality holds {\em if and only if}
\begin{equation}\label{eq:key2}
\bx = \bb = \Hk{\bb}.
\end{equation}
To see this, let $\Omega$ be the support set of $\Hk{\bb}$ and $\overline{\Omega}$ be the complement. Let $\bb_1 = \PO{\bb} = \Hk{\bb}$ and $\bb_2 = \mathcal{P}_{\overline{\Omega}}(\bb)$. Likewise, we define $\bx_1$ and $\bx_2$ as the components of $\bx$ supported on $\Omega$ and $\overline{\Omega}$ respectively. Hence, \eqref{eq:key1} indicates $\bx_1 = \bb_1$ and $\bx_2 = (1+ \gamma^{-1})\bb_2$ where we assume $\gamma > 0$ since $\gamma=0$ immediately implies $\Hk{\bb} = \bb$ and hence the equality of~\eqref{eq:old} does not hold. If $\zeronorm{\bb_1} < k$, then we have $\bx_2 = \bb_2 = \bzero$ since $\bb_1$ contains the $k$ largest absolute elements of $\bb$. Otherwise, the fact that $\zeronorm{\bx} \leq k$ and $\bx_1 = \bb_1$ implies $\bx_2 = \bzero$, and hence $\bb_2$. Therefore, we obtain~\eqref{eq:key2}.

When~\eqref{eq:key2} happens, however, we in reality have $\twonorm{\Hk{\bb} - \bx} = \twonorm{\bb - \bx} = 0$. In other words, the factor of $2$ in~\eqref{eq:old} can essentially be replaced with an {\em arbitrary constant}! In this sense, we conclude that the bound~\eqref{eq:old} is not tight. Our new estimate for hard thresholding is as follows:
\begin{theorem}[Tight Bound for Hard Thresholding]\label{thm:key}
Let $\bb \in \Rd$ be an arbitrary vector and $\bx \in \Rd$ be any $K$-sparse signal. For any $k \geq K$, we have the following bound:
\begin{equation*}
	\twonorm{\Hk{\bb} - \bx} \leq \sqrt{\nu} \twonorm{\bb - \bx},\quad \nu = 1 + \frac{\rho + \sqrt{\(4 + \rho \) \rho } }{2}, \quad \rho =  \frac{\min\{K, d-k\}}{k - K + \min\{K, d-k\}}.
\end{equation*}
In particular, our bound is tight in the sense that there exist specific vectors of $\bb$ and $\bx$ such that the equality holds.
\end{theorem}
\begin{remark}[Maximum of $\nu$]
In contrast to the constant bound~\eqref{eq:old}, our result asserts that the deviation resulting from hard thresholding is inversely proportional to $\sqrt{k}$~(when $K \leq d-k$) in a universal manner. When $k$ tends to $d$, $\rho$ is given by $(d-k)/(d-K)$ which is still decreasing with respect to $k$. Thus, the maximum value of $\rho$ equals one. Even in this case, we find that $\sqrt{\nu_{\max}} = \sqrt{1 + \frac{\sqrt{5} + 1}{2}} = \frac{\sqrt{5} + 1}{2} \approx 1.618$.
\end{remark}
\begin{remark}
Though for some batch algorithms such as IHT and CoSaMP, the constant bound~\eqref{eq:old} suffices to establish the convergence due to specific conditions, we show in Section~\ref{sec:alg} that it cannot ensure the global convergence for stochastic algorithms.
\end{remark}
\begin{remark}
{When $\bx$ is not exactly $K$-sparse, we still can bound the error by  $\twonorm{\Hk{\bb} - \bx} \leq\twonorm{\Hk{\bb} - \Hk{\bx}}+ \twonorm{\Hk{\bx} - \bx}$. Thus, without loss of generality, we assumed that the signal $\bx$ is $K$-sparse.}
\end{remark}
\begin{proof}(Sketch) Our bound follows from fully exploring the sparsity pattern of the signals and from fundamental arguments in optimization. Denote
\begin{equation*}
\bw \defeq \Hk{\bb}.
\end{equation*}
Let $\Omega$ be the support set of $\bw$ and let $\overline{\Omega}$ be its complement. We immediately have $\PO{\bb} = \bw$. Let $\Omega'$ be the support set of $\bx$. Define
\begin{equation*}
\bb_1 = \mathcal{P}_{\Omega \backslash \Omega'}\( \bb \),\quad \bb_2 = \mathcal{P}_{\Omega \cap \Omega'}\( \bb \),\quad \bb_3 = \mathcal{P}_{\overline{\Omega} \backslash \Omega'}\( \bb \),\quad \bb_4 = \mathcal{P}_{\overline{\Omega} \cap \Omega'}\( \bb \).
\end{equation*}
Likewise, we define $\bx_i$ and $\bw_i$ for $1 \leq i \leq 4$. Due to the construction, we have $\bw_1 = \bb_1, \bw_2 = \bb_2, \bw_3 = \bw_4 = \bx_1 = \bx_3 = \bzero$. Our goal is to estimate the maximum value of $\twonorm{\bw - \bx}^2 / \twonorm{\bb - \bx}^2$. It is easy to show that when attaining the maximum, $\twonorm{\bb_3}$ must be zero. Denote
\begin{equation}\label{eq:tmp_t}
\gamma \defeq \frac{\twonorm{\bw - \bx}^2}{\twonorm{\bb - \bx}^2} =  \frac{\twonorm{\bb_1}^2 + \twonorm{\bb_2 - \bx_2}^2 + \twonorm{\bx_4}^2}{\twonorm{\bb_1}^2 + \twonorm{\bb_2 - \bx_2}^2  + \twonorm{\bb_4 - \bx_4}^2}.
\end{equation}
Note that the variables here only involve $\bx$ and $\bb$. Arranging the equation we obtain
\begin{equation}\label{eq:key3}
(\gamma-1) \twonorm{\bb_2 - \bx_2}^2 + \gamma \twonorm{\bb_4 - \bx_4}^2 - \twonorm{\bx_4}^2 + (\gamma-1) \twonorm{\bb_1}^2 = 0.
\end{equation}
It is evident that for specific choices of $\bb$ and $\bx$, we have $\gamma = 1$. Since we are interested in the maximum of $\gamma$, we assume $\gamma > 1$ below. Fixing $\bb$, we can view the left-hand side of the above equation as a function of $\bx$. One may verify that the function has a positive definite Hessian matrix and thus it attains the minimum at stationary point given by
\begin{equation}\label{eq:tmp_stp}
\bx_2^* = \bb_2,\quad \bx_4^* = \frac{\gamma}{\gamma-1} \bb_4.
\end{equation}
On the other hand,~\eqref{eq:key3} implies that the minimum function value should not be greater than zero. Plugging the stationary point back gives
\begin{equation*}
\twonorm{\bb_1}^2 \gamma^2 - (2\twonorm{\bb_1}^2 + \twonorm{\bb_4}^2) \gamma + \twonorm{\bb_1}^2 \leq 0.
\end{equation*}
Solving the above inequality with respect to $\gamma$, we obtain
\begin{equation}\label{eq:tmp_test}
\gamma \leq 1 + {\(2\twonorm{\bb_1}^2\)}^{-1}{\(\twonorm{\bb_4}^2 + \sqrt{\(4\twonorm{\bb_1}^2+\twonorm{\bb_4}^2\) \twonorm{\bb_4}^2}\)}.
\end{equation}
To derive an upper bound that is uniform over the choice of $\bb$, we recall that $\bb_1$ contains the largest absolute elements of $\bb$ while $\bb_4$ has smaller values. In particular, the average in $\bb_1$ is larger than that in $\bb_4$, which gives
\begin{equation*}
{\twonorm{\bb_4}^2}/{\zeronorm{\bb_4}} \leq {\twonorm{\bb_1}^2}/{\zeronorm{\bb_1}}.
\end{equation*}
Note that $\zeronorm{\bb_1} = k - \zeronorm{\bb_2} = k - (K - \zeronorm{\bb_4})$. Hence, combining with the fact that $0 \leq \zeronorm{\bb_4} \leq \min\{K, d-k\}$ and optimizing over $\zeronorm{\bb_4}$ in the above inequality gives
\begin{equation}\label{eq:tmp_b4}
\twonorm{\bb_4}^2 \leq \frac{\min\{K, d-k\}}{k - K + \min\{K, d-k\}} \twonorm{\bb_1}^2.
\end{equation}
Finally, we arrive at a uniform upper bound
\begin{equation*}
\gamma \leq 1 + \frac{\rho + \sqrt{\(4 + \rho \) \rho } }{2},\quad \rho =  \frac{\min\{K, d-k\}}{k - K + \min\{K, d-k\}}.
\end{equation*}
See Appendix~\ref{sec:app:proofkey} for the full proof.
\end{proof}
\begin{remark}[Tightness]
We construct proper vectors $\bb$ and $\bx$ to establish the tightness of our bound by a backward induction. Note that $\gamma$ equals $\nu$ if and only if $\twonorm{\bb_4}^2 = \rho \twonorm{\bb_1}^2$. Hence, we pick
\begin{equation}\label{eq:tmp_choice}
\twonorm{\bb_4}^2 = \rho \twonorm{\bb_1}^2,\quad \bx_2 = \bb_2,\quad \bx_4 = \frac{\nu}{\nu - 1} \bb_4,
\end{equation}
where $\bx_2$ and $\bx_4$ are actually chosen as the stationary point as in~\eqref{eq:tmp_stp}. We note that the quantity of $\nu$ only depends on $d$, $k$ and $K$, not on the components of $\bb$ or $\bx$. Plugging the above back to~\eqref{eq:tmp_t} justifies $\gamma = \nu$.

It remains to show that our choices in~\eqref{eq:tmp_choice} do not violate the definition of $\bb_i$'s, i.e., we need to ensure that the elements in $\bb_1$ or $\bb_2$ are equal to or greater than those in $\bb_3$ or $\bb_4$. Note that there is no such constraint for the $K$-sparse vector $\bx$. Let us consider the case $K < d -k$ and $\zeronorm{\bb_4} = K$, so that $\zeronorm{\bb_1} = k$ and $\rho = K / k$. Thus, the first equality of~\eqref{eq:tmp_choice} holds as soon as all the entries of $\bb$ have same magnitude. The fact $\zeronorm{\bb_4} = K$ also implies $\Omega'$ is a subset of $\overline{\Omega}$ due to the definition of $\bb_4$ and the sparsity of $\bx$, hence we have $\bx_2 = \bzero = \bb_2$. Finally, picking $\bx_4$ as we did in~\eqref{eq:tmp_choice} completes the reasoning since it does not violate the sparsity constraint on $\bx$.
\end{remark}

{As we pointed out and just verified, the bound given by Theorem~\ref{thm:key} is tight. However, if there is additional information for the signals, a better bound can be established. For instance, let us further assume that the signal $\bb$ is $r$-sparse. If $r \leq k$, then $\bb_4$ is a zero vector and~\eqref{eq:tmp_test} reads as $\gamma \leq 1$. Otherwise, we have $\zeronorm{\bb_4} \leq \min\{ K, r - k \}$ and~\eqref{eq:tmp_b4} is improved to
\begin{equation*}
\twonorm{\bb_4}^2 \leq \frac{\min\{K, r-k\}}{k - K + \min\{K,  r-k\}} \twonorm{\bb_1}^2.
\end{equation*}
Henceforth, we can show that the parameter $\rho$ is given by
\begin{equation*}
\rho =  \frac{\min\{K, r-k\}}{k - K + \min\{K, r-k\}}.
\end{equation*}
Note that the fact $r \leq d$ implies that the above is a tighter bound than the one in Theorem~\ref{thm:key}.

We would also like to mention that in Lemma~1 of~\cite{jain2014iterative}, a closely related bound was established:
\begin{equation}\label{eq:jain_bound}
\twonorm{\Hk{\bb} - \bb} \leq \sqrt{\frac{d - k}{d - K}} \twonorm{\bb - \bx}.
\end{equation}
One may use this nice result to show that
\begin{equation}\label{eq:triangle_jain}
\twonorm{\Hk{\bb} - \bx} \leq \twonorm{\Hk{\bb} - \bb} + \twonorm{\bb - \bx} \leq \( 1 + \sqrt{\frac{d - k}{d - K}} \)\twonorm{\bb - \bx},
\end{equation}
which also improves on~\eqref{eq:old} provided $k > K$. However, one shortcoming of~\eqref{eq:triangle_jain} is that the factor depends on the dimension. For comparison, we recall that in the regime $K \leq d - k$, our bound is free of the dimension. This turns out to be a salient feature to integrate hard thresholding into stochastic methods, and we will comment on it more in Section~\ref{sec:alg}.}

\section{Implications to Compressed Sensing}\label{sec:imp}
In this section, we investigate the implications of Theorem~\ref{thm:key} for compressed sensing and signal processing.  Since most of the HT-based algorithms utilize the deviation bound~\eqref{eq:old} to derive the convergence condition, they can be improved by our new bound. We exemplify the power of our theorem on two popular algorithms: IHT~\cite{blumensath2009iterative} and CoSaMP \cite{needell2009cosamp}. We note that our analysis also applies to their extensions such as \cite{bahmani2013greedy}. {To be clear, the purpose of this section is not dedicated to improving the best RIP condition for which recovery is possible by any methods (either convex or non-convex). Rather, we focus on two broadly used greedy algorithms and illustrate how our bound improves on previous results.}

We proceed with a brief review of the problem setting in compressed sensing. Compressed sensing algorithms aim to recover the true $K$-sparse signal $\bx^* \in \Rd$ from a set of its (perhaps noisy) measurements
\begin{equation}\label{model:CS}
\by = \bA \bx^* + \beps,
\end{equation}
where $\beps \in \Rd$ is some observation noise and $\bA$ is a known $n \times d$ sensing matrix with $n \ll d$, hence the name compressive sampling. In general, the model is not identifiable since it is an under-determined system. Yet, the prior knowledge that $\bx^*$ is sparse radically changes the premise. That is, if the geometry of the sparse signal is preserved under the action of the sampling matrix $\bA$ for a restricted set of directions, then it is possible to invert the sampling process. Such a novel idea was quantified as the $k$th restricted isometry property of $\bA$ by~\cite{candes2005decoding}, which requires that there exists a constant $\delta \geq 0$, such that for all $k$-sparse signals $\bx$
\begin{equation}\label{eq:ric}
(1 - \delta)\twonorm{\bx}^2 \leq \twonorm{\bA \bx}^2 \leq (1+\delta)\twonorm{\bx}^2.
\end{equation}
The $k$th restricted isometry constant~(RIC) $\delta_k$ is then defined as the smallest one that satisfies the above inequalities. Note that $\delta_{2k} < 1$ is the minimum requirement for distinguishing all $k$-sparse signals from the measurements. This is because for two arbitrary $k$-sparse vectors $\bx_1$ and $\bx_2$ and their respective measurements $\by_1$ and $\by_2$, the RIP condition reads as
\begin{equation*}
(1-\delta_{2k}) \twonorm{\bx_1 - \bx_2}^2 \leq \twonorm{\by_1 - \by_2}^2 \leq (1+\delta_{2k}) \twonorm{\bx_1 - \bx_2}^2,
\end{equation*}
for which $\delta_{2k} < 1$ guarantees that $\bx_1 \neq \bx_2$ implies $\by_1 \neq \by_2$. To date, there are three quintessential examples known to exhibit a profound restricted isometry behavior as long as the number of measurements is large enough: Gaussian matrices~(optimal RIP, i.e., very small $\delta_k$), partial Fourier matrices~(fast computation) and Bernoulli ensembles~(low memory footprint). Notably, it was shown in recent work that random matrices with a heavy-tailed distribution also satisfy the RIP with overwhelming probability~\cite{adamczak2011restricted,li2014compressed}.

Equipped with the standard RIP condition, many efficient algorithms have been developed. A partial list includes $\ell_1$-norm based convex programs, IHT, CoSaMP, SP and regularized OMP \cite{needell2010signal}, along with much interesting work devoted to improving or sharpening the RIP condition~\cite{wang2012recovery,mo2012remark,cai2013sharp,mo2015sharp}. To see why relaxing RIP is of central interest, note that the standard result \cite{baraniuk2008simple} asserts that the RIP condition $\delta_k \leq \delta$ holds with high probability over the draw of $\bA$ provided
\begin{equation}\label{eq:rip n}
n \geq \const_0 \delta^{-2} k \log (d/k).
\end{equation}
Hence, a slight relaxation of the condition $\delta_k \leq \delta$ may dramatically decrease the number of measurements. {That being said, since the constant $\const_0$ above is unknown, in general one cannot tell the precise sample size for greedy algorithms. Estimating the constant is actually the theme of phase transition~\cite{donoho2010precise,donoho2013accurate}. While precise phase transition for $\ell_1$-based convex programs has been well understood~\cite{wainwright2009sharp}, an analogous result for greedy algorithms remains an open problem. Notably, in~\cite{blanchard2015performance}, phase transition for IHT/CoSaMP was derived using the constant bound~\eqref{eq:old}. We believe that our tight bound shall sharpen these results and we leave it as our future work. In the present paper, we focus on the ubiquitous RIP condition. In the language of RIP, we establish improved results.}

\subsection{Iterative Hard Thresholding}
The IHT algorithm recovers the underlying $K$-sparse signal $\bx^*$ by iteratively performing a full gradient descent on the least-squares loss followed by a hard thresholding step. That is, IHT starts with an arbitrary point $\bx^0$ and at the $t$-th iteration, it updates the new solution as follows:
\begin{equation}\label{eq:x_IHT}
\bx^{t} = \Hk{\bx^{t-1} + \bA\trans(\by - \bA \bx^{t-1})}.
\end{equation}
Note that~\cite{blumensath2009iterative} used the parameter $k = K$. However, in practice one may only know to an upper bound on the true sparsity $K$. Thus, we consider the projection sparsity $k$ as a parameter that depends on $K$. To establish the global convergence with a geometric rate of $0.5$, \cite{blumensath2009iterative} applied the bound~\eqref{eq:old} and assumed the RIP condition
\begin{equation}
\delta_{2k+K} \leq 0.18.
\end{equation}
As we have shown, \eqref{eq:old} is actually not tight and hence, their results, especially the RIP condition can be improved by Theorem~\ref{thm:key}.

\begin{theorem}\label{thm:iht}
Consider the model~\eqref{model:CS} and the IHT algorithm~\eqref{eq:x_IHT}. Pick $k \geq K$ and let $\{ \bx^t \}_{t\geq 1}$ be the iterates produced by IHT. Then, under the RIP condition $\delta_{ 2k+K} \leq 1/\sqrt{8\nu}$, for all $t \geq 1$
\begin{equation*}
\twonorm{\bx^t - \bx^*} \leq 0.5^t \twonorm{\bx^{0} - \bx^*} + \const \twonorm{\beps},
\end{equation*}
where $\nu$ is given by Theorem~\ref{thm:key}.
\end{theorem}

Let us first study the vanilla case $k = K$. \cite{blumensath2009iterative} required $\delta_{3K} \leq 0.18$ whereas our analysis shows $\delta_{3K} \leq 0.22$ suffices. Note that even a little relaxation on RIP is challenging and may require several pages of mathematical induction~\cite{candes2008restricted,cai2010new,foucart2012sparse}. In contrast, our improvement comes from a direct application of Theorem~\ref{thm:key} which only modifies several lines of the original proof in~\cite{blumensath2009iterative}. See Appendix~\ref{sec:app:proofimp} for details. In view of~\eqref{eq:rip n}, we find that the necessary number of measurements for IHT is dramatically reduced with a factor of $0.67$ by our new theorem in that the minimum requirement of $n$ is inversely proportional to the square of $\delta_{ 2k+K}$.

Another important consequence of the theorem is a characterization on the RIP condition and the sparsity parameter, which, to the best of our knowledge, has not been studied in the literature. In \cite{blumensath2009iterative}, when gradually tuning $k$ larger than $K$, it always requires $\delta_{2k+K} \leq 0.18$. Note that due to the monotonicity of RIC, i.e., $\delta_r \leq \delta_{r'}$ if $r \leq r'$, the condition turns out to be more and more stringent. Compared to their result, since $\nu$ is inversely proportional to $\sqrt{k}$, Theorem~\ref{thm:iht} is powerful especially when $k$ becomes larger. For example, suppose $k = 20 K$. In this case, Theorem~\ref{thm:iht} justifies that IHT admits the linear convergence as soon as $\delta_{41K} \leq 0.32$ whereas \cite{blumensath2009iterative} requires $\delta_{41K} \leq 0.18$. Such a property is appealing in practice, in that among various real-world applications, the true sparsity is indeed unknown and we would like to estimate a conservative upper bound on it.

On the other hand, for a {given} sensing matrix, there does exist a fundamental limit for the maximum choice of $k$. To be more precise, the condition in Theorem~\ref{thm:iht} together with the probabilistic argument~\eqref{eq:rip n} require
\begin{equation*}
1/\sqrt{8\nu} \geq \delta_{2k+K},\quad \const_1 \nu (2k+K) \log\(d/(2k+K)\) \leq n.
\end{equation*}
Although it could be very interesting to derive a quantitative characterization for the maximum value of $k$, we argue that it is perhaps intractable owing to two aspects: First, it is known that one has to enumerate all the combinations of the $2k+K$ columns of $\bA$ to  compute the restricted isometry constant $\delta_{2k+K}$~\cite{bah2010improved,bah2014bounds}. This suggests that it is NP-hard to estimate the largest admissible value of $k$. Also, there is no analytic solution of the stationary point for the left-hand side of the second inequality.

\subsection{Compressive Sampling Matching Pursuit}
The CoSaMP algorithm proposed by~\cite{needell2009cosamp} is one of the most efficient algorithms for sparse recovery. Let $F(\bx) = \twonorm{\by - \bA \bx}^2$. CoSaMP starts from an arbitrary initial point $\bx^0$ and proceeds as follows:
\begin{align*}
\Omega^t &= \supp{\nabla F(\bx^{t-1}),\ k} \cup \supp{\bx^{t-1}},\\
\bb^t &= \argmin_{\bx}\ F(\bx),\ \st\ \supp{\bx} \subset \Omega^t,\\
\bx^t &= \Hk{\bb^t}.
\end{align*}
Compared to IHT which performs hard thresholding after gradient update, CoSaMP prunes the gradient at the beginning of each iteration, followed by solving a least-squares program restricted on a small support set. In particular, in the last step, CoSaMP applies hard thresholding to form a $k$-sparse iterate for future updates. The analysis of CoSaMP consists of bounding the estimation error in each step. Owing to Theorem~\ref{thm:key}, we advance the theoretical result of CoSaMP by improving the error bound for its last step, and hence the RIP condition.

\begin{theorem}\label{thm:cosamp}
Consider the model~\eqref{model:CS} and the CoSaMP algorithm. Pick $k \geq K$ and let $\{ \bx^t \}_{t\geq 1}$ be the iterates produced by CoSaMP. Then, under the RIP condition
\begin{equation*}
\delta_{ 3k+K} \leq \frac{\(\sqrt{32\nu + 49} - 9\)^{1/2}}{4\sqrt{\nu-1}},
\end{equation*}
it holds that for all $t \geq 1$
\begin{equation*}
\twonorm{\bx^t - \bx^*} \leq 0.5^t \twonorm{\bx^{0} - \bx^*} + \const \twonorm{\beps},
\end{equation*}
where $\nu$ is given by Theorem~\ref{thm:key}.
\end{theorem}

Roughly speaking, the bound is still inversely proportional to $\sqrt{\nu}$. Hence, it is monotonically increasing with respect to $k$, indicating our theorem is more effective for a large quantity of $k$. In fact, for the CoSaMP algorithm, our bound above is superior to the best known result even when $k=K$. To see this, we have the RIP condition $\delta_{4K} \leq 0.31$. In comparison, \cite{needell2009cosamp} derived a bound $\delta_{4K} \leq 0.1$ and~\cite[Theorem 6.27]{foucart2013mathematical} improved it to $\delta_{4K} < 0.29$ for a geometric rate of $0.5$. We notice that for binary sparse vectors, \cite{jain2014iterative} presented a different proof technique and obtained the RIP condition $\delta_{4K} \leq 0.35$ for CoSaMP.

%

\section{Hard Thresholding in Large-Scale Optimization}\label{sec:alg}
Now we move on to the machine learning setting where our focus is pursuing an optimal sparse solution that minimizes a given objective function based on a set of training samples $Z_1^n \defeq \{Z_i\}_{i=1}^n$. Different from compressed sensing, we usually have sufficient samples which means $n$ can be very large. Therefore, the computational complexity is of primary interest. Formally, we are interested in optimizing the following program:
\begin{equation}\label{eq:primal}
\min_{\bx \in \Rd}\ F(\bx; Z_1^n) = \frac{1}{n}\sum_{i=1}^{n} f(\bx; Z_i),\quad \st\ \zeronorm{\bx} \leq K,\ \twonorm{\bx} \leq \omega.
\end{equation}
The global optimum of the above problem is denoted by $\xopt$. We note that the objective function is presumed to be decomposable with respect to the samples. This is quite a mild condition and most of the popular machine learning models fulfill it. Typical examples include (but not limited to) the sparse linear regression and sparse logistic regression:
\begin{itemize}
\item Sparse Linear Regression: For all $1 \leq i \leq n$, we have $Z_i = (\ba_i, y_i) \in \Rd \times \R$ and the loss function $F(\bx; Z_1^n) = \frac{1}{2n} \twonorm{\bA \bx - \by}^2$ is the least-squares and can be explained by $f(\bx; Z_i) = \frac{1}{2}\twonorm{\ba_i \cdot \bx - y_i}^2$.

\item Sparse Logistic Regression: For all $1 \leq i \leq n$, we have $Z_i = (\ba_i, y_i) \in \Rd \times \{+1, -1\}$ and the negative log-likelihood is penalized, i.e., $F(\bx; Z_1^n) = \frac{1}{n} \sum_{i=1}^{n} \log\( 1 + \exp\(-y_i \ba_i \cdot \bx \) \)$ for which $f(\bx; Z_i) = \log\( 1 + \exp\(-y_i \ba_i \cdot \bx \) \)$.
\end{itemize}
To ease notation, we will often write $F(\bx; Z_1^n)$ as $F(\bx)$ and $f(\bx; Z_i)$ as $f_i(\bx)$ for $i=1, 2, \cdots, n$. It is worth mentioning that the objective function $F(\bx)$ is allowed to be non-convex. Hence, in order to ensure the existence of a global optimum, a natural option is to impose an $\ell_p$-norm ($p \geq 1$) constraint~\cite{loh2012high,loh2015regularized}. Here we choose the $\ell_2$-norm constraint owing to its fast projection. Previous work, e.g.,~\cite{agarwal2012fast} prefers the computationally less efficient $\ell_1$-norm to promote sparsity and to guarantee the existence of optimum. In our problem, yet, we already have imposed the hard sparsity constraint so the $\ell_2$-norm constraint is a better fit.

The major contribution of this section is a computationally efficient algorithm termed hard thresholded stochastic variance reduced gradient method~(HT-SVRG) to optimize~\eqref{eq:primal}, tackling one of the most important problems in large-scale machine learning: producing sparse solutions by stochastic methods. We emphasize that the formulation~\eqref{eq:primal} is in stark contrast to the $\ell_1$-regularized programs considered by previous stochastic solvers such as Prox-SVRG~\cite{xiao2014proximal} and SAGA~\cite{saga}. We target here a stochastic algorithm for the {\em non-convex} problem that is less exploited in the literature. From a theoretical perspective,~\eqref{eq:primal} is more difficult to analyze but it always produces sparse solutions, whereas performance guarantees for convex programs are fruitful but one cannot characterize the sparsity of the obtained solution (usually the solution is not sparse). When we appeal to stochastic algorithms to solve the convex programs, the $\ell_1$-norm formulation becomes much less effective in terms of sparsification, naturally owing to the randomness. See~\cite{langford2009sparse,xiao2010dual,duchi2009fobos} for more detailed discussion on the issue. We also remark that existing work such as~\cite{yuan2018gradient,bahmani2013greedy,jain2014iterative} investigated the sparsity-constrained problem~\eqref{eq:primal} in a batch scenario, which is not practical for large-scale learning problems. The perhaps most related work to our new algorithm is~\cite{nguyen2014linear}. Nonetheless, the optimization error therein does not vanish for noisy statistical models.

Our main result shows that for prevalent statistical models, our algorithm is able to recover the true parameter with a linear rate. Readers should distinguish the optimal solution $\xopt$ and the true parameter. For instance, consider the model~\eqref{model:CS}. Minimizing~\eqref{eq:primal} does not amount to recovering $\bx^*$ if there is observation noise. In fact, the convergence to $\xopt$ is only guaranteed to an accuracy reflected by the {\em statistical precision} of the problem, i.e., $\twonorm{\bx^* - \xopt}$, which is the best one can hope for any statistical model~\cite{agarwal2012fast}. We find that the global convergence is attributed to both the tight bound and the variance reduction technique to be introduced below, and examining the necessity of them is an interesting future work. 

\begin{algorithm}[h]
\caption{Hard Thresholded Stochastic Variance Reduced Gradient Method~(HT-SVRG)}
\label{alg:all}
\begin{algorithmic}[1]
\REQUIRE Training samples $\{Z_i\}_{i=1}^n$, maximum stage count $S$, sparsity parameter $k$, update frequency $m$, learning rate $\eta$, radius $\omega$, initial solution $\xtilde^0$.
\ENSURE Optimal solution $\xtilde^S$.
\FOR{$s=1$ to $S$}
\STATE Set $\xtilde = \xtilde^{s-1}$, $\ \mutilde = \frac{1}{n} \sum_{i=1}^{n} \nabla f_i(\xtilde)$, $\ \bx^0 = \xtilde$.
\FOR{$t=1$ to $m$}
\STATE Uniformly pick $i_t \in \{1, 2, \cdots, n\}$ and update the solution
\begin{align*}
\bb^{t} &= \bx^{t-1} - \eta \( \nabla f_{i_t}(\bx^{t-1}) - \nabla f_{i_t}(\xtilde) + \mutilde \),\\
\br^{t} &= \Hk{\bb^{t}},\\
\bx^{t} & = \Pi_{\omega}(\br^t).
\end{align*}
\ENDFOR
\STATE Uniformly choose $j^s \in \{0, 1, \cdots, m-1\}$ and set $\xtilde^{s} = \bx^{j^s}$.
\ENDFOR
\end{algorithmic}
\end{algorithm}

\subsection{Algorithm}
Our algorithm~(Algorithm~\ref{alg:all}) applies the framework of~\cite{svrg}, where the primary idea is to leverage past gradients for the current update for the sake of variance reduction~--~a technique that has a long history in statistics~\cite{owen2000safe}. To guarantee that each iterate is $k$-sparse, it then invokes the hard thresholding operation. Note that the orthogonal projection for $\br^t$ will not change the support set, and hence $\bx^t$ is still $k$-sparse. Also note that our sparsity constraint in~\eqref{eq:primal} reads as $\zeronorm{\bx} \leq K$. What we will show below is that when the parameter $k$ is properly chosen (which depends on $K$), we obtain a globally convergent sequence of iterates. 

The most challenging part on establishing the global convergence comes from the hard thresholding operation $\Hk{\br^t}$. Note that it is $\bb^t$ that reduces the objective value in expectation. If $\bb^t$ is not $k$-sparse (usually it is dense), $\bx^t$ is not equal to $\bb^t$ so it does not decrease the objective function. In addition, compared with the convex proximal operator~\cite{saga} which enjoys the non-expansiveness of the distance to the optimum, the hard thresholding step can enlarge the distance up to a multiple of $2$ if using the bound~\eqref{eq:old}. What makes it a more serious issue is that these inaccurate iterates $\bx^t$ will be used for future updates, and hence the error might be progressively propagated at an exponential rate.

Our key idea is to first bound the curvature of the function from below and above to establish RIP-like condition, which, combined with Theorem~\ref{thm:key}, downscales the deviation resulting from hard thresholding. Note that $\nu$ is always greater than one (see Theorem~\ref{thm:key}), hence the curvature bound is necessary. Due to variance reduction, we show that the optimization error vanishes when restricted on a small set of directions as soon as we have sufficient samples. Moreover, with hard thresholding we are able to control the error per iteration and to obtain near-optimal sample complexity.

\subsection{Deterministic Analysis}
We will first establish a general theorem that characterizes the progress of HT-SVRG for approximating an arbitrary $K$-sparse signal $\xhat$. Then we will discuss how to properly choose the hyper-parameters of the algorithm. Finally we move on to specify $\xhat$ to develop convergence results for a global optimum of~\eqref{eq:primal} and for a true parameter (e.g., $\bx^*$ of the compressed sensing problem).

\subsubsection{Assumption}
Our analysis depends on two properties of the curvature of the objective function that have been standard in the literature. Readers may refer to~\cite{bickel2009simultaneous,negahban2009unified,jain2014iterative} for a detailed description.
\begin{definition}[Restricted Strong Convexity]
A differentiable function $g:\ \Rd \rightarrow \R$ is said to satisfy the property of restricted strong convexity~(RSC) with parameter $\alpha_r > 0$, if for all vectors $\bx$, $\bx' \in \Rd$ with $\zeronorm{\bx - \bx'} \leq r$, it holds that
\begin{equation*}
g(\bx') - g(\bx) - \inner{ \nabla g(\bx)}{ \bx' - \bx } \geq \fractwo{\alpha_r} \twonorm{\bx' - \bx}^2.
\end{equation*}
\end{definition}
\begin{definition}[Restricted Smoothness]
A differentiable function $g:\ \Rd \rightarrow \R$ is said to satisfy the property of restricted smoothness~(RSS) with parameter $L_r > 0$, if for all vectors $\bx$, $\bx' \in \Rd$ with $\zeronorm{\bx - \bx'} \leq r$, it holds that
\begin{equation*}
\twonorm{\nabla g(\bx') - \nabla g(\bx)} \leq L_r \twonorm{\bx' - \bx}.
\end{equation*}
\end{definition}
With these definitions, we assume the following:
\begin{enumerate}[label=$(A\arabic*)$]
\item $F(\bx)$ satisfies the RSC condition with parameter $\alpha_{k+K}$.\label{as:rsc}
\item For all $1 \leq i \leq n$, $f_i(\bx)$ satisfies the RSS condition with parameter $L_{3k+K}$.\label{as:rss}
\end{enumerate}
Here, we recall that $K$ was first introduced in~\eqref{eq:primal} and the parameter $k$ was used in our algorithm. Compared to the convex algorithms such as SAG~\cite{sag}, SVRG~\cite{svrg} and SAGA~\cite{saga} that assume strong convexity and smoothness everywhere, we only assume these in a restricted sense. This is more practical especially in the high dimensional regime where the Hessian matrix could be degenerate~\cite{agarwal2012fast}. We also stress that the RSS condition is imposed on each $f_i(\bx)$, whereas prior work requires it for $F(\bx)$ which is milder than ours~\cite{negahban2009unified}.

\subsubsection{Upper Bound of Progress}
For brevity, let us denote
\begin{equation*}
L := L_{3k+K},\quad	\alpha := \alpha_{k+K}, \quad c := L / \alpha,
\end{equation*}
where we call the quantity $c$ as the condition number of the problem. It is also crucial to measure the $\ell_2$-norm of the gradient restricted on sparse directions, and we write
\begin{equation*}
\twonorm{\nabla_{3k+K} F(\bx)} := \max_{\Omega}\big\{ \twonorm{\PO{\nabla F(\bx)}}:\ \abs{\Omega} \leq 3k+K \big\}.
\end{equation*}
Note that for convex programs, the above evaluated at a global optimum is zero. As will be clear, $\twonorm{\nabla_{3k+K} F(\bx)}$ reflects how close the iterates returned by HT-SVRG can be to the point $\bx$. For prevalent statistical models, it vanishes when there are sufficient samples. Related to this quantity, our analysis also involves
\begin{equation*}
Q(\bx) := \(16 \nu \eta^2 L \omega m + \frac{2\omega}{\alpha}\) \twonorm{\nabla_{3k+K} F(\bx)} + 4 \nu \eta^2 m \twonorm{\nabla_{3k+K} F(\bx)}^2,
\end{equation*}
where we recall that $\nu$ is the expansiveness factor given by Theorem~\ref{thm:key}, $\eta$ and $m$ are used in the algorithm and $\omega$ is a universal constant that upper bounds the $\ell_2$-norm of the signal we hope to estimate. Virtually, with an appropriate parameter setting, $Q(\bx)$ scales as $\twonorm{\nabla_{3k+K} F(\bx)}$ which will be clarified. For a particular stage $s$, we denote $\mathcal{I}^s := \{ i_1, i_2, \cdots, i_m \}$, i.e., the samples randomly chosen for updating the solution.
\begin{theorem}\label{thm:general}
Consider Algorithm~\ref{alg:all} and a $K$-sparse signal $\xhat$ of interest. Assume~\ref{as:rsc} and~\ref{as:rss}. Pick the step size $0 < \eta < 1/(4L)$. If $\nu < 4L/(4L-\alpha)$, then it holds that
\begin{equation*}
\E \big[F(\xtilde^{s}) - F(\xhat)\big] \leq \beta^s\big[F(\xtilde^0) - F(\xhat)\big] + \tau(\xhat),
\end{equation*}
where the expectation is taken over $\{ \mathcal{I}^1, j^1, \mathcal{I}^2, j^2, \cdots, \mathcal{I}^s, j^s \}$ and $0 < \beta < 1$ provided that $m$ is large enough. In particular, for $1/(1-\eta\alpha) < \nu < 4L/(4L-\alpha)$, we have
\begin{align*}
\beta &= \beta_1 \defeq \frac{1}{\(2\nu \eta \alpha - 2 \nu \eta^2 \alpha L - \nu +1\)m} + \frac{2\nu\eta^2 \alpha L}{2\nu \eta \alpha - 2 \nu \eta^2 \alpha L - \nu +1},\\
\tau(\xhat) &= \tau_1(\xhat) \defeq  \frac{\alpha Q(\xhat) }{2(2\nu \eta \alpha - 2\nu \eta^2 \alpha L - \nu + 1)(1 - \beta_1)m}.
\end{align*}
For $\nu \leq 1/(1-\eta\alpha)$, we have
\begin{equation*}
\beta =	\beta_2 \defeq \frac{1}{\nu \eta \alpha (1 - 2\eta L)m} + \frac{2\eta L}{1 - 2\eta L},\quad \tau(\xhat) = \tau_2(\xhat) \defeq \frac{Q(\xhat)}{2\nu \eta \alpha (1 - 2\eta L)(1-\beta_2)m}.
\end{equation*}
\end{theorem}
The proof can be found in Appendix~\ref{sec:app:proofgeneral}.
\begin{remark}\label{rmk:online}
For the theorem to hold, $\sqrt{\nu} < \sqrt{4L/(4L - \alpha)} \leq \sqrt{4/3} \approx 1.15$ due to $L \geq \alpha$. Hence, the conventional bound~\eqref{eq:old} is not applicable. In contrast, Theorem~\ref{thm:key} asserts that this condition can be fulfilled by tuning $k$ slightly larger than $K$.
\end{remark}
\begin{remark}
With the conditions on $\eta$ and $\nu$, the coefficient $\beta$ is always less than one provided that $m$ is sufficiently large.
\end{remark}
\begin{remark}
The theorem does {\em not} assert convergence to an arbitrary sparse vector $\xhat$. This is because $F(\xtilde^s) - F(\xhat)$ might be less than zero. However, specifying $\xhat$ does give convergence results, as to be elaborated later.
\end{remark}

\subsubsection{Hyper-Parameter Setting}

Before moving on to the convergence guarantee, let us discuss the minimum requirement on the hyper-parameters $k$, $m$ and $\eta$, and determine how to choose them to simplify Theorem~\ref{thm:general}.

For the sake of success of HT-SVRG, we require $\nu < 4c/(4c-1)$, which implies $\rho < 1/(16c^2 - 4c)$. Recall that $\rho$ is given in Theorem~\ref{thm:key}. In general, we are interested in the regime $K \leq k \ll d$. Hence, we have $\rho = K/k$ and the minimum requirement for the sparsity parameter is 
\begin{equation}\label{eq:k req}
k > (16c^2-4c)K.
\end{equation}
To our knowledge, the idea of relaxed sparsity was first introduced in \cite{zhang2011sparse} for OMP and in \cite{jain2014iterative} for projected gradient descent. However, the relaxed sparsity here emerges in a different way in that HT-SVRG is a stochastic algorithm, and their proof technique cannot be used.

We also contrast our tight bound to the inequality~\eqref{eq:triangle_jain} that is obtained by combining the triangle inequality and Lemma~1 of~\cite{jain2014iterative}. Following our proof pipeline, \eqref{eq:triangle_jain} gives
\begin{equation*}
k \geq \( 1 - \( \sqrt{{4c}(4c-1)^{-1} }-1 \)^2 \) d + \( \sqrt{{4c}(4c-1)^{-1}} -1\)^2 K
\end{equation*}
which grows with the dimension $d$, whereas using Theorem~\ref{thm:key} the sparsity parameter $k$ depends only on the desired sparsity $K$. In this regard, we conclude that for the stochastic case, our bound is vital.

Another component of the algorithm is the update frequency $m$. Intuitively, HT-SVRG performs $m$ number of stochastic gradient update followed by a full gradient evaluation, in order to mitigate the variance. In this light, $m$ should not be too small. Otherwise, the algorithm reduces to the full gradient method which is not computationally efficient. On the other spectrum, a large $m$ leads to a slow convergence that is reflected in the convergence coefficient $\beta$. To quantitatively analyze how $m$ should be selected,  let us consider the case $\nu \leq 1/(1-\eta\alpha)$ for example. The case $1/(1-\eta\alpha) < \nu < 4L/(4L-\alpha)$ follows in a similar way. In order to ensure $\beta_2 < 1$, we must have $m > 1 / \(\nu \eta \alpha (1-4\eta L)\)$. In particular, picking
\begin{equation}\label{eq:eta req}
\eta = \frac{\eta'}{L},\quad \eta' \in (0, 1/4),
\end{equation}
we find that the update frequency $m$ has to satisfy
\begin{equation}\label{eq:m req}
m > \frac{c}{\nu \eta' (1 - \eta')},
\end{equation}
which is of the same order as in the convex case~\cite{svrg} when $\eta' = \Theta(1)$. Note that the way we choose the learning rate $\eta = \eta' / L$ is also a common practice in convex optimization~\cite{nesterov2004introductory}.

With~\eqref{eq:k req},~\eqref{eq:eta req} and~\eqref{eq:m req} in mind, we provide detailed choices of the hyper-parameters. Due to $0 < \eta < 1/(4L)$, $\beta_1$ is monotonically increasing with respect to $\nu$. By Theorem~\ref{thm:key}, we know that $\nu$ is decreasing with respect to $k$. Thus, a larger quantity of $k$ results in a smaller value of $\beta_1$, and hence a faster rate. Interestingly, for $\beta_2$ we discover that the smaller the $k$ is, the faster the algorithm concentrates. Hence, we have the following consequence:
\begin{proposition}\label{prop:lambda}
Fix $\eta$ and $m$. Then the optimal choice of $\nu$ in Theorem~\ref{thm:general} is $\nu = 1/(1-\eta \alpha)$ in the sense that the convergence coefficient $\beta$ attains the minimum.
\end{proposition}

In light of the proposition, in the sections to follow, we will only consider the setting $\nu = 1/(1-\eta \alpha)$. But we emphasize that our analysis and results essentially apply to any $\nu \leq 4L/(4L - \alpha)$.

Now let
\begin{equation}\label{eq:param}
\eta = \frac{1}{8L},\quad m = 4(8c - 1),\quad k = 8c(8c-1) K.
\end{equation}
This gives
\begin{equation}\label{eq:beta tau}
\beta = \frac{2}{3},\quad \tau(\xhat) = \frac{5\omega}{\alpha} \twonorm{\nabla_{3k+K} F(\xhat)} + \frac{1}{\alpha L} \twonorm{\nabla_{3k+K} F(\xhat)}^2.
\end{equation}

\subsubsection{Global Linear Convergence}
We are in the position to state the global linear convergence to an optimum of the sparsity-constrained optimization program~\eqref{eq:primal}.
\begin{corollary}\label{coro:optimum}
Assume~\ref{as:rsc} and~\ref{as:rss}. Consider the HT-SVRG algorithm with hyper-parameters given in~\eqref{eq:param}. Then the sequence $\{\xtilde^s\}_{s\geq 1}$ converges linearly to a global optimum $\xopt$ of~\eqref{eq:primal}
\begin{align*}
\E \big[F(\xtilde^{s}) - F(\xopt)\big] \leq&\ \(\frac{2}{3}\)^s\big[F(\xtilde^0) - F(\xopt)\big]  \\
&\ + \frac{5\omega}{\alpha} \twonorm{\nabla_{3k+K} F(\xopt)}+ \frac{1}{\alpha L} \twonorm{\nabla_{3k+K} F(\xopt)}^2.
\end{align*}
\end{corollary}
\begin{proof}
This is a direct consequence of Theorem~\ref{thm:general}.
\end{proof}
Whenever $\nabla_{3k+K} F(\xopt) = \bzero$, the corollary reads as
\begin{equation*}
\E \big[F(\xtilde^{s}) - F(\xopt)\big] \leq \(\frac{2}{3}\)^s\big[F(\xtilde^0) - F(\xopt)\big].
\end{equation*}
It implies that if one is solving a convex problem without the sparsity constraint but the optimal solution happens to be sparse, it is safe to perform hard thresholding without loss of optimality. We exemplify such behavior with another algorithm SAGA~\cite{saga} in Appendix~\ref{sec:app:saga}. In the  noiseless compressed sensing setting where $\by = \bA \bx^*$, the corollary guarantees that HT-SVRG exactly recovers the underlying true signal $\bx^*$ when $F(\bx)$ is chosen as the least-squares loss in that $\xopt = \bx^*$ and $\nabla F(\bx^*) = \bA\trans (\bA\bx^* - \by) = \bzero$.

On the other side, the RSC property implies that
\begin{equation*}
\twonorm{\xtilde^s - \xhat} \leq \sqrt{ \frac{2 \max\{ F(\xtilde^s) - F(\xhat), 0 \}}{\alpha} } + \frac{2 \twonorm{\nabla_{k+K} F(\xhat)}}{\alpha}. 
\end{equation*}
The proof is straightforward and can be found in Lemma 14 of~\cite{shen2017iteration}. Now we specify $\xhat$ as the true parameter of some statistical model, for instance, $\bx^*$ in~\eqref{model:CS}. It is hence possible to establish recovery guarantee of $\bx^*$, which is known as the problem of parameter estimation.
\begin{corollary}\label{coro:param}
Assume~\ref{as:rsc} and~\ref{as:rss}. Let $L'$ be the RSS parameter of $F(\bx)$ at the sparsity level $3k + K$. Consider the HT-SVRG algorithm with hyper-parameters given in~\eqref{eq:param}. Then the sequence $\{\xtilde^s\}_{s\geq 1}$ recovers a $K$-sparse signal $\bx^*$ with a geometric rate
\begin{align*}
\E \big[\twonorm{\xtilde^s - \bx^*} \big] \leq&\ \sqrt{\frac{2L'}{\alpha}} \cdot \( \frac{2}{3} \)^{\frac{s}{2}} \twonorm{\xtilde^0 - \bx^*} + \sqrt{\frac{10\omega}{\alpha^2} \twonorm{\nabla_{3k+K} F(\bx^*)}} \\
&\ + \( \sqrt{\frac{2}{\alpha^3}}  + \frac{3}{\alpha} \) \twonorm{\nabla_{3k+K} F(\bx^*)}.
\end{align*}
\end{corollary}
The proof can be found in Appendix~\ref{sec:app:proofparam}.
\begin{remark}
The RSS parameter $L'$ of $F(\bx)$ always ranges in $[\alpha, L]$, which is simply by definition.
\end{remark}

\subsubsection{Computational Complexity}

We compare the computational complexity of HT-SVRG to that of projected gradient descent~(PGD) studied in \cite{jain2014iterative}, which is a batch counterpart to HT-SVRG. First, we remark that the analysis of PGD is based on the smoothness parameter $L'$ of $F(\bx)$ at sparsity level $2k+K$. We write $c' = L' / \alpha$. To achieve a given accuracy $\epsilon > 0$, PGD requires $\order{c' \log(1/\epsilon)}$ iterations. Hence the total computational complexity is $\order{nc' d\log(1/\epsilon)}$. For HT-SVRG, in view of Corollary~\ref{coro:optimum}, the convergence coefficient is a constant. Hence, HT-SVRG needs $\order{\log(1/\epsilon)}$ iterations where we note that the error term $\twonorm{\nabla_{3k+K} F(\bx^*)}$ can be made as small as $\epsilon$ with sufficient samples (to be clarified in the sequel). In each stage, HT-SVRG computes a full gradient $\mutilde$ followed by $m$ times stochastic updates. Therefore, the total complexity of HT-SVRG is given by $\order{(n+c)d \log(1/\epsilon)}$ by noting the fact $m = \order{c}$. In the scenario $c < n(c' - 1)$, HT-SVRG significantly improves on PGD in terms of time cost.

\subsection{Statistical Results}

The last ingredient of our theorem is the term $\tau(\xhat)$ which measures how close the iterates could be to a given sparse signal $\xhat$. With appropriate hyper-parameter settings, the quantity relies exclusively on $\twonorm{\nabla_{3k+K} F(\xhat)}$, as suggested by~\eqref{eq:beta tau}. Thereby, this section is dedicated to characterizing $\twonorm{\nabla_{3k+K} F(\xhat)}$. We will also give examples for which HT-SVRG is computationally more efficient than PGD. For the purpose of a concrete result, we study two problems: sparse linear regression and sparse logistic regression. These are two of the most popular statistical models in the literature and have found a variety of applications in machine learning and statistics~\cite{raskutti2011minimax}. Notably, it is known that similar statistical results can be built for low-rank matrix regression, sparse precision matrix estimation, as suggested in~\cite{negahban2009unified,agarwal2012fast}.

\subsubsection{Sparse Linear Regression}
For sparse linear regression, the observation model is given by
\begin{equation}\label{model:linear}
\by = \bA \bx^* + \beps,\quad \zeronorm{\bx^*} \leq K,\ \twonorm{\bx^*} \leq \omega,
\end{equation}
where $\bA \in \Rnd$ is the design matrix, $\by \in \Rn$ is the response, $\beps \in \Rn$ is some noise, and $\bx^*$ is the $K$-sparse true parameter we hope to estimate from the knowledge of $\bA$ and $\by$. Note that when we have the additional constraint $n \ll d$, the model above is exactly that of compressed sensing~\eqref{model:CS}.

In order to (approximately) estimate the parameter, a natural approach is to optimize the following non-convex program:
\begin{equation}\label{prog:linear}
\min_{\bx}\ F(\bx) \defeq \frac{1}{2n}\sum_{i=1}^{n}\twonorm{y_i - \ba_i \cdot \bx}^2,\quad \st\ \zeronorm{\bx} \leq K,\ \twonorm{\bx} \leq \omega.
\end{equation}

For our analysis, we assume the following on the design matrix and the noise:
\begin{enumerate}[label=$(A\arabic*)$, start=3]
\item $\ba_1, \ba_2, \dots, \ba_n$ are independent and identically distributed~(i.i.d.) Gaussian random vectors $N(\bzero, \bSigma)$. All the diagonal elements of $\bSigma$ satisfy $\Sigma_{jj} \leq 1$. The noise $\beps$ is independent of $\bA$ and its entries are i.i.d. Gaussian random variables ${N}(0, \sigma^2)$.  \label{as:sub}
\end{enumerate}

\begin{proposition}\label{prop:stat linear}
Consider the sparse linear regression model~\eqref{model:linear} and the program~\eqref{prog:linear}. Assume~\ref{as:sub}. Then for a sparsity level $r$, 
\begin{itemize}
\item with probability at least $1 - \exp(-\const_0 n)$,
\begin{equation*}
\alpha_r = \lambda_{\min}(\bSigma) - \const_1 \frac{r \log d}{n},\quad L_r' = \lambda_{\max}(\bSigma) + \const_2 \frac{r \log d}{n};
\end{equation*}

\item with probability at least $1 - {\const_3 r}/{d}$
\begin{equation*}
L_r = \const_4 r \log d;
\end{equation*}

\item and with probability at least $1 - \const_5 / d$
\begin{equation*}
\twonorm{\nabla_r F(\bx^*)} \leq \const_6 \sigma \sqrt{\frac{r \log d}{n}},\quad \twonorm{\nabla_r F(\xopt)} \leq L_r'  \twonorm{\xopt - \bx^*}  + \const_6 \sigma \sqrt{\frac{ r \log d}{ n}}.
\end{equation*}
\end{itemize}
Above, $\lambda_{\min}(\bSigma)$ and $\lambda_{\max}(\bSigma)$ are the minimum and maximum singular values of $\bSigma$ respectively.
\end{proposition}
We recall that $\alpha_r$ and $L_r$ are involved in our assumptions~\ref{as:rsc} and~\ref{as:rss}, and $L'_r$ is the RSS parameter of $F(\bx)$. The estimation for $\alpha_r$, $L'_r$ and $\twonorm{\nabla_r F(\bx^*)}$ follows from standard results in the literature~\cite{raskutti2011minimax}, while that for $L_r$ follows from Proposition E.1 in~\cite{bellec2016slope} by noting the fact that bounding $L_r$ amounts to estimating $\max_i\twonorm{\mathcal{H}_r(\ba_i)}^2$. In order to estimate $\twonorm{\nabla_r F(\xopt)}$, notice that
\begin{align*}
\twonorm{\nabla_r F(\xopt)} \leq&\ \twonorm{\nabla_r F(\xopt) - \nabla_r F(\bx^*)} + \twonorm{ \nabla_r F(\bx^*) } \\
\leq&\ \twonorm{\nabla F(\xopt) - \nabla F(\bx^*)} + \twonorm{ \nabla_r F(\bx^*) } \\
\leq&\ L_r' \twonorm{\xopt - \bx^*} + \twonorm{ \nabla_r F(\bx^*) },
\end{align*}
where we use the definition of RSS in the last inequality.

Now we let $r = 3k + K = \textrm{const} \cdot {c^2 K}$ and get $\alpha = \lambda_{\min}(\bSigma) - \const_1 \frac{c^2 K \log d}{n}$, $L = \const_4 c^2 K \log d$. Suppose that $\lambda_{\min}(\bSigma) =  2 \const_4 (K \log d)^2$ and $n = q \cdot \frac{\const_1}{\const_4} K \log d$ with $q \geq 1$. Then our assumptions~\ref{as:rsc} and~\ref{as:rss} are met with high probability with
\begin{equation*}
\alpha = \const_4 (K \log d)^2, \ L= \const_4 (K \log d)^3,\ \text{and}\  c = K \log d.
\end{equation*}

For Corollary~\ref{coro:optimum}, as far as
\begin{equation*}
s \geq \const_7 \log\( \frac{F(\xtilde^0) - F(\xopt)}{\epsilon} \),\ n = \const_7 \( \omega \sigma  \)^2 \epsilon^{-2}  K \log d,
\end{equation*}
we have
\begin{equation*}
\E\big[ F(\xtilde^s) - F(\xopt) \big] \leq \epsilon + \frac{\lambda_{\max}(\bSigma)}{\lambda_{\min}(\bSigma)} \twonorm{\xopt - \bx^*} +\( \frac{\lambda_{\max}(\bSigma)}{\lambda_{\min}(\bSigma)}\twonorm{\xopt - \bx^*} \)^2
\end{equation*}
for some accuracy parameter $\epsilon > 0$. This suggests that it is possible for HT-SVRG to approximate a global optimum of~\eqref{eq:primal} up to $\twonorm{\xopt - \bx^*}$, namely the statistical precision of the problem.

Returning to Corollary~\ref{coro:param}, to guarantee that
\begin{equation*}
\E\big[ \twonorm{\xtilde^s - \bx^*} \big] \leq \epsilon,
\end{equation*}
it suffices to pick
\begin{equation*}
s \geq \const_8 \log(\omega \sqrt{c'} / \epsilon),\quad n = \const_8 (\omega \sigma)^2 \epsilon^{-4} { K \log d }.
\end{equation*}

Finally, we compare the computational cost to PGD. It is not hard to see that under the same situation $\lambda_{\min}(\bSigma) =  2 \const_4 (K \log d)^2$ and $n = \frac{\const_1}{\const_4} K \log d$,
\begin{equation*}
L' = \const_4 (K \log d)^3,\ c'= K \log d,\ \text{provided\ that}\ \lambda_{\max}(\bSigma) = \const_4 (K \log d)^3 - \frac{\const_2\const_4}{\const_1} (K \log d)^2.
\end{equation*}
Thus $c < n(c' - 1)$, i.e., HT-SVRG is more efficient than PGD. It is also possible to consider other regimes of the covariance matrix and the sample size, though we do not pursue it here.

\subsubsection{Sparse Logistic Regression}

For sparse logistic regression, the observation model is given by
\begin{equation}
\Pr(y_i \mid \ba_i;\ \bx^*) = \frac{1}{1+ \exp(-y_i \ba_i \cdot \bx^*)},\quad \zeronorm{\bx^*} \leq K,\ \twonorm{\bx} \leq \omega,\  \forall\ 1 \leq i \leq n,
\end{equation}
where $y_i$ is either $0$ or $1$. It then learns the parameter by minimizing the negative log-likelihood:
\begin{equation}\label{eq:logistic}
\min_{\bx}\ F(\bx) \defeq \frac{1}{n} \sum_{i=1}^{n} \log\( 1 + \exp(-y_i \ba_i \cdot \bx) \),\quad \st\ \zeronorm{\bx} \leq K,\ \twonorm{\bx} \leq \omega.
\end{equation}
There is a large body of work showing that the statistical property is rather analogous to that of linear regression. See, for example, \cite{negahban2009unified}. In fact, the statistical results apply to generalized linear models as well.

\subsection{A Concurrent Work}
After we posted the first version~\cite{shen2016tight} on {arXiv}, \cite{li2016stochastic} made their work public where a similar algorithm to HT-SVRG was presented. Their theoretical analysis applies to convex objective functions while we allow the function $F(\bx)$ to be non-convex. We also fully characterize the convergence behavior of the algorithm by showing the trade-off between the sparsity parameter $k$ and the convergence coefficient $\beta$~(Proposition~\ref{prop:lambda}).

\section{Experiments}\label{sec:exp}
In this section, we present a comprehensive empirical study  for the proposed HT-SVRG algorithm on two tasks: sparse recovery (compressed sensing) and image classification. The experiments on sparse recovery is dedicated to verifying the theoretical results we presented, and we visualize the classification models learned by HT-SVRG to demonstrate the practical efficacy.

\subsection{Sparse Recovery}
To understand the practical behavior of our algorithm as well as to justify the theoretical analysis, we perform experiments on synthetic data. The experimental settings are as follows:
\begin{itemize}
\item {\bf Data Generation.} \ 
The data dimension $d$ is fixed as $256$ and we generate an $n\times d$ Gaussian random sensing matrix $\bA$ whose entries are i.i.d. with zero mean and variance $1/n$. Then $1000$ $K$-sparse signals $\bx^*$ are independently generated, where the support of each signal is uniformly chosen. That is, we run our algorithm and the baselines for $1000$ trials. The measurements $\by$ for each signal $\bx^*$ is obtained by $\by = \bA \bx^*$ which is noise free. In this way, we are able to study the convergence rate by plotting the logarithm of the objective value since the optimal objective value is known to be zero.

\item {\bf Baselines.} \ 
We mainly compare with two closely related algorithms: IHT and PGD. Both of them compute the full gradient of the least-squares loss followed by hard thresholding. Yet, PGD is more general, in the sense that it allows the sparsity parameter $k$ to be larger than the true sparsity $K$ ($k=K$ for IHT) and also considers a flexible step size $\eta$ ($\eta = 1$ for IHT). Hence, PGD can be viewed as a batch counterpart to our method HT-SVRG.

\item {\bf Evaluation Metric.} \ 
We say a signal $\bx^*$ is successfully recovered by a solution $\bx$ if 
\begin{equation*}
\frac{\twonorm{\bx - \bx^*}}{\twonorm{\bx^*}} < 10^{-3}.
\end{equation*}
In this way, we can compute the percentage of success over the $1000$ trials for each algorithm.

\item {\bf Hyper-Parameters.} \ 
If not specified, we use $m = 3n$, $k = 9K$, and $S=10000$ for HT-SVRG. We also use the heuristic step size $\eta = 2/\mathrm{svds}(\bA \bA\trans)$ for HT-SVRG and PGD, where $\mathrm{svds}(\bA \bA\trans)$ returns the largest singular value of the matrix $\bA\bA\trans$. Since for each stage, HT-SVRG computes the full gradient for $(2m/n + 1)$ times, we run the IHT and PGD for $(2m/n + 1)S$ iterations for fair comparison, i.e., all of the algorithms have the same number of full gradient evaluations.
\end{itemize}

\subsubsection{Phase Transition}
Our first simulation aims at offering a big picture on the recovery performance. To this end, we vary the number of measurements $n$ from $1$ to $256$, roughly with a step size $8$. We also study the performance with respect to the true sparsity parameter $K$, which ranges from $1$ to $26$, roughly with step size $2$. The results are illustrated in Figure~\ref{fig:recovery}, where a brighter block means a higher percentage of success and the brightest ones indicate exact sparse recovery. It is apparent that PGD and HT-SVRG require fewer measurements for an accurate recovery than IHT, possibly due to the flexibility in choosing the sparsity parameter and the step size. We also observe that as a stochastic algorithm, HT-SVRG performs comparably to PGD. This suggests that HT-SVRG is an appealing solution to large-scale sparse learning problems in that HT-SVRG is computationally more efficient.
\begin{figure*}[h]
\centering
\includegraphics[width=0.3\linewidth]{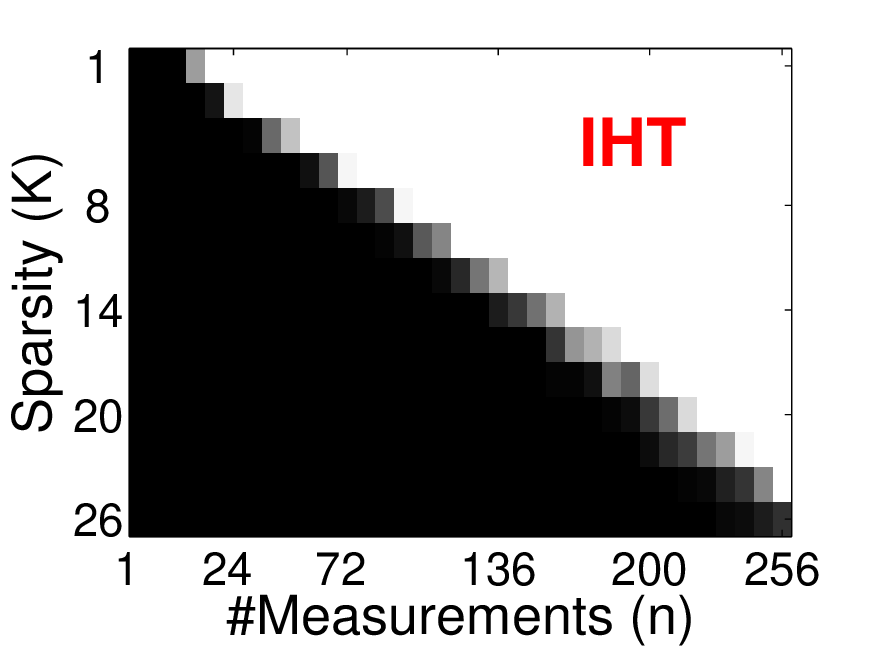}
\includegraphics[width=0.3\linewidth]{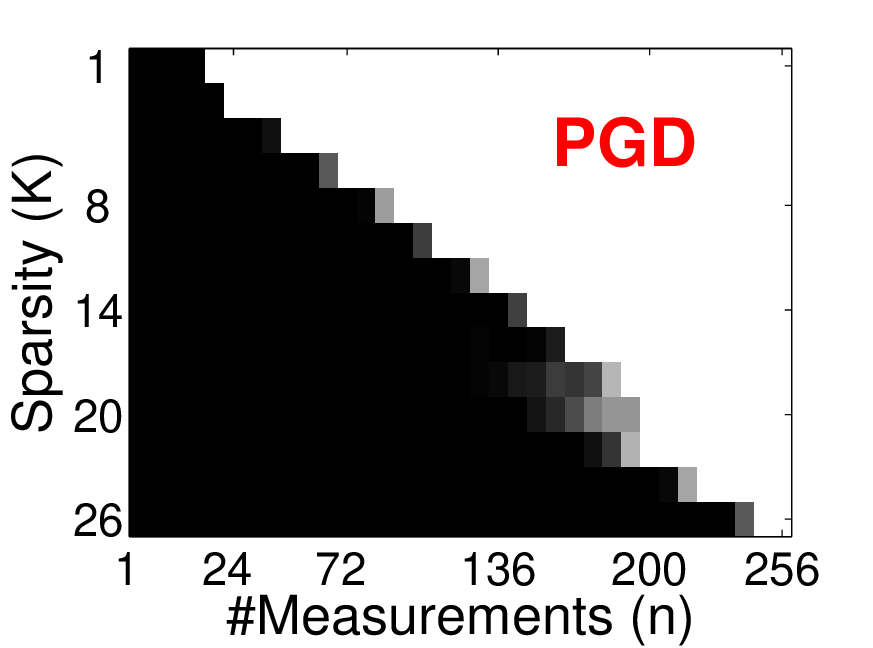}
\includegraphics[width=0.3\linewidth]{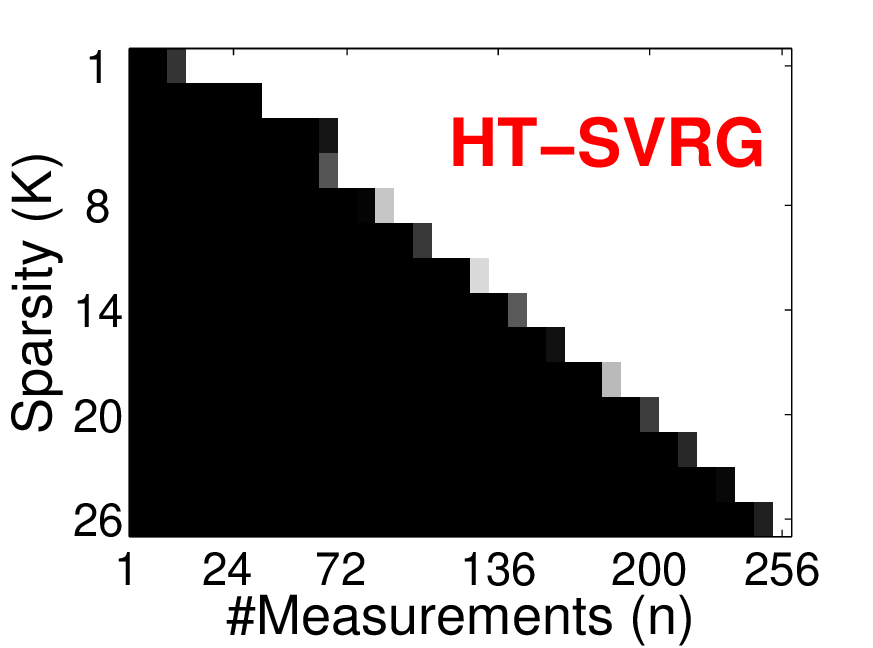}
\caption{{\bf Percentage of successful recovery under various sparsity and sample size.} The values range from $0$ to $100$, where a brighter color  means a higher percentage of success (the brightest blocks correspond to the value of $100$). PGD admits a higher percentage of recovery compared to IHT because it flexibly chooses the step size and sparsity parameter. As a stochastic variant, HT-SVRG performs comparably to the batch counterpart PGD.}
\label{fig:recovery}
\end{figure*}

In Figure~\ref{fig:recoverydetail}, we exemplify some of the results obtained from HT-SVRG by plotting two kinds of curves: the success of percentage against the sample size $n$ and that against the signal sparsity $K$. In this way, one can examine the detailed values and can determine the minimum sample size for a particular sparsity. For instance, the left panel tells that to ensure that $80\%$ percents of the $16$-sparse signals are recovered, we have to collect $175$ measurements. We can also learn from the right panel that using $232$ measurements, any signal whose sparsity is $22$ or less can be reliably recovered.
\begin{figure*}[h]
\centering
\includegraphics[width=0.35\linewidth]{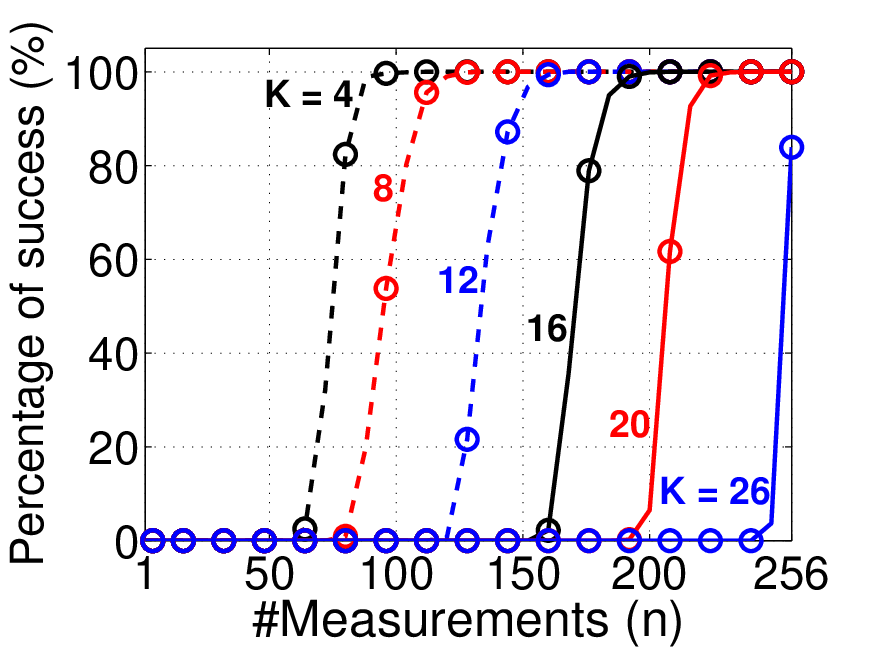}
\includegraphics[width=0.35\linewidth]{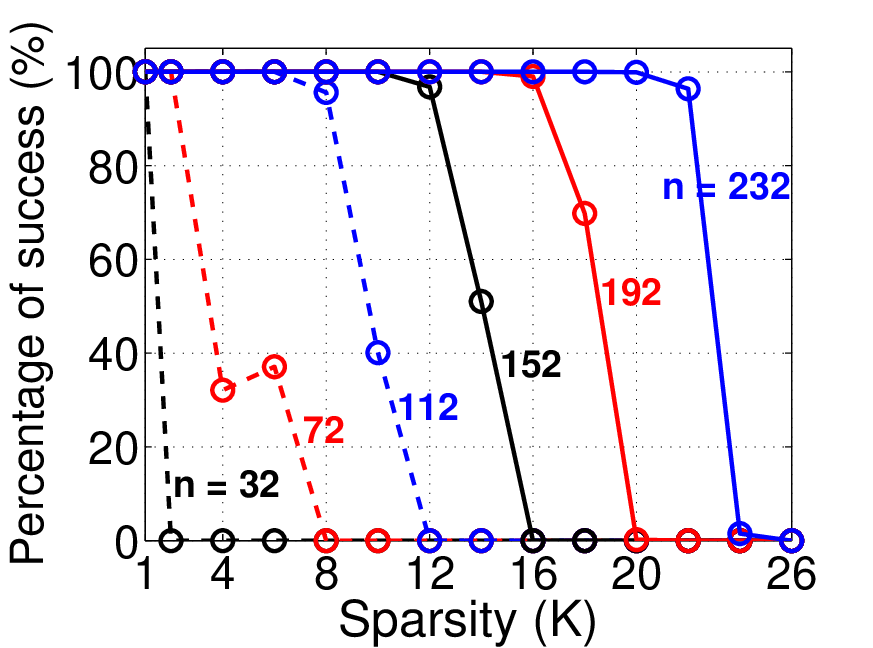}
\caption{{\bf Percentage of success of HT-SVRG against the number of measurements (left) and the sparsity (right).}}
\label{fig:recoverydetail}
\end{figure*}

Based on the results in Figure~\ref{fig:recovery} and Figure~\ref{fig:recoverydetail}, we have an approximate estimation on the minimum requirement of the sample size which ensures accurate (or exact) recovery. Now we are to investigate how many measurements are needed to guarantee a success percentage of $95\%$ and $99\%$. To this end, for each signal sparsity $K$, we look for the number of measurements $n_0$ from Figure~\ref{fig:recovery} where $90$ percents of success are achieved. Then we carefully enlarge $n_0$ with step size $1$ and run the algorithms. The empirical results are recorded in Figure~\ref{fig:min_n}, where the circle markers represent the empirical results with different colors indicating different algorithms, e.g., red circle for empirical observation of HT-SVRG. Then we fit these empirical results by linear regression, which are plotted as solid or dashed lines. For example, the green line is a fitted model for IHT. We find that $n$ is almost linear with $K$. Especially, the curve of HT-SVRG is nearly on top of that of PGD, which again verifies HT-SVRG is an attractive alternative to the batch method.
\begin{figure*}[h]
\centering
\includegraphics[width=0.42\linewidth]{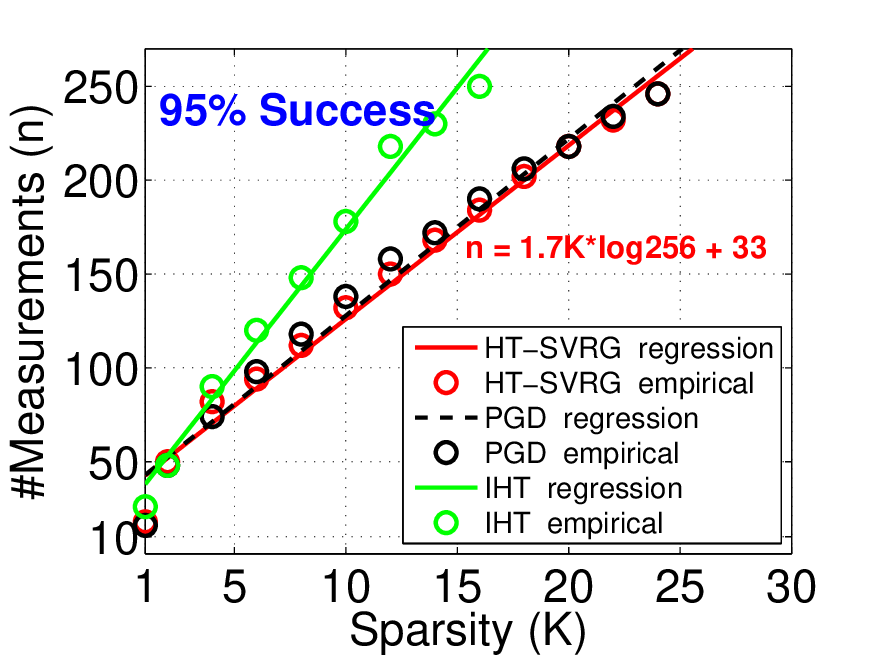}
\includegraphics[width=0.42\linewidth]{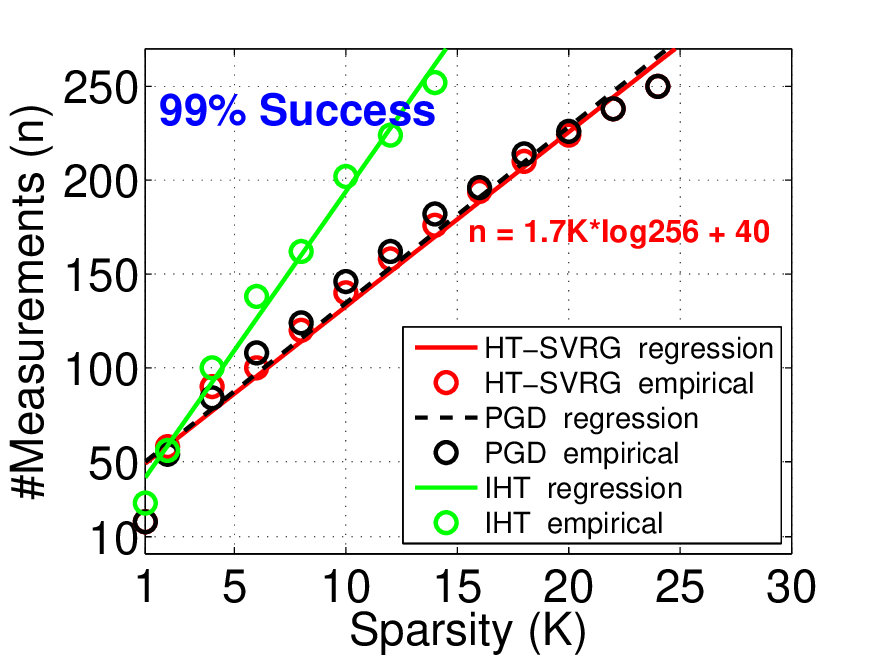}
\caption{{\bf Minimum number of measurements to achieve $95\%$ and $99\%$ percentage of success.} Red equation indicates the linear regression of HT-SVRG. The markers and curves for HT-SVRG are almost on top of PGD, which again justifies that HT-SVRG is an appealing stochastic alternative to the batch method PGD.}
\label{fig:min_n}
\end{figure*}

\subsubsection{Influence of Hyper-Parameters}

Next, we turn to investigate the influence of the hyper-parameters, i.e., the sparsity parameter $k$, update frequency $m$ and step size $\eta$ on the convergence behavior of HT-SVRG. We set the true sparsity $K = 4$ and collect $100$ measurements for each groundtruth signal, i.e., $n = 100$. Note that the standard setting we employed is $k = 9K = 36$, $m = 3n = 300$ and $\eta = 2/\mathrm{svds}(\bA \bA\trans) \approx 0.3$. Each time we vary one of these parameters while fixing the other two, and the results are plotted in Figure~\ref{fig:parameter}. We point out that although the convergence result~(Theorem~\ref{thm:general}) is deterministic, the vanishing optimization error~(Proposition~\ref{prop:stat linear}) is guaranteed under a probabilistic argument. Hence, it is possible that for a specific configuration of parameters, $97\%$ of the signals are exactly recovered but HT-SVRG fails on the remaining, as we have observed in, e.g., Figure~\ref{fig:recoverydetail}. Clearly, we are not supposed to average all the results to examine the convergence rate. For our purpose, we set a threshold $95\%$, that is, we average over the success trials if more than $95\%$ percents of the signals are exactly recovered. Otherwise, we say that the set of parameters cannot ensure convergence and we average over these failure signals which will give an illustration of divergence.

\begin{figure*}[h]
\centering
\includegraphics[width=0.32\linewidth]{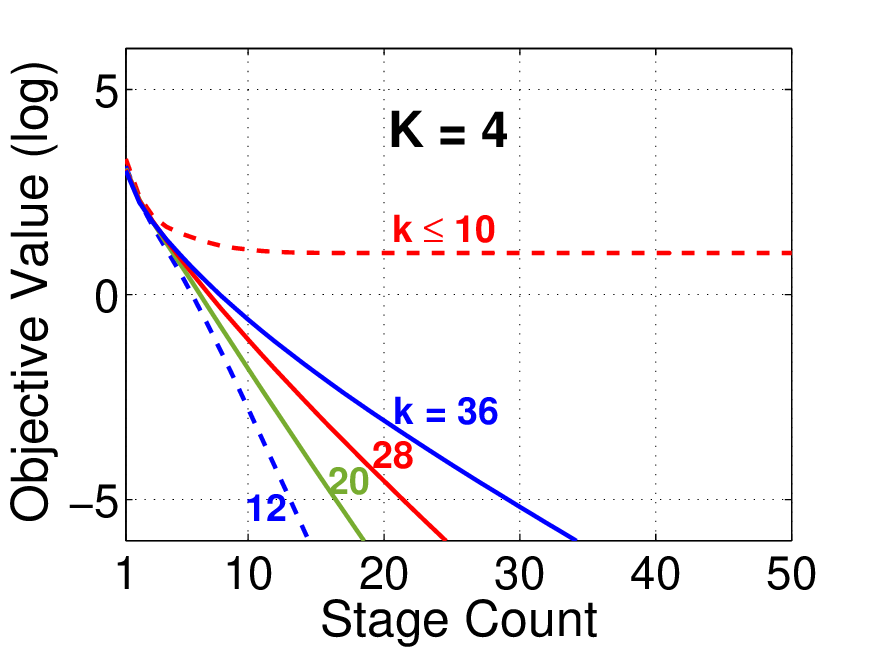}
\includegraphics[width=0.32\linewidth]{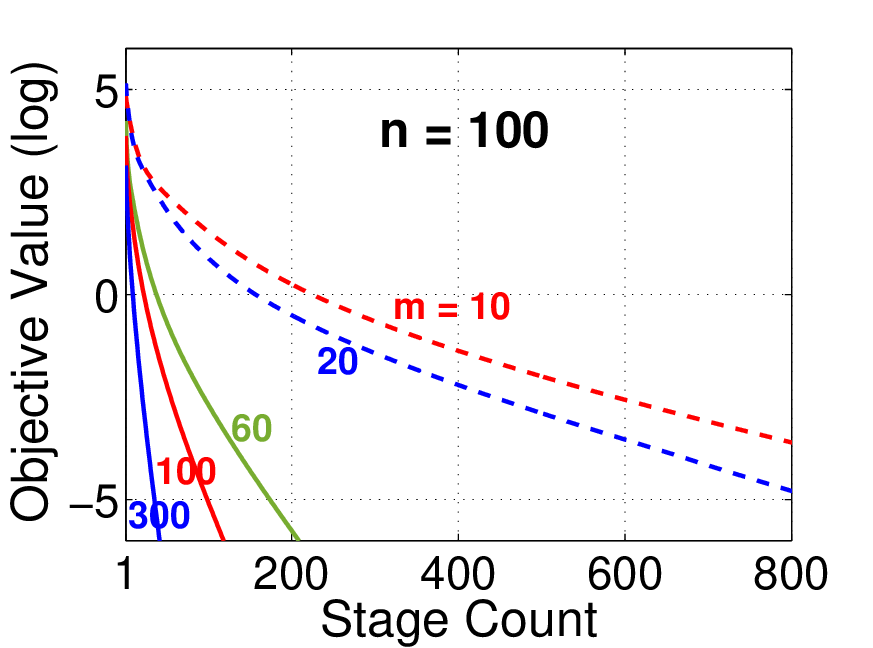}
\includegraphics[width=0.32\linewidth]{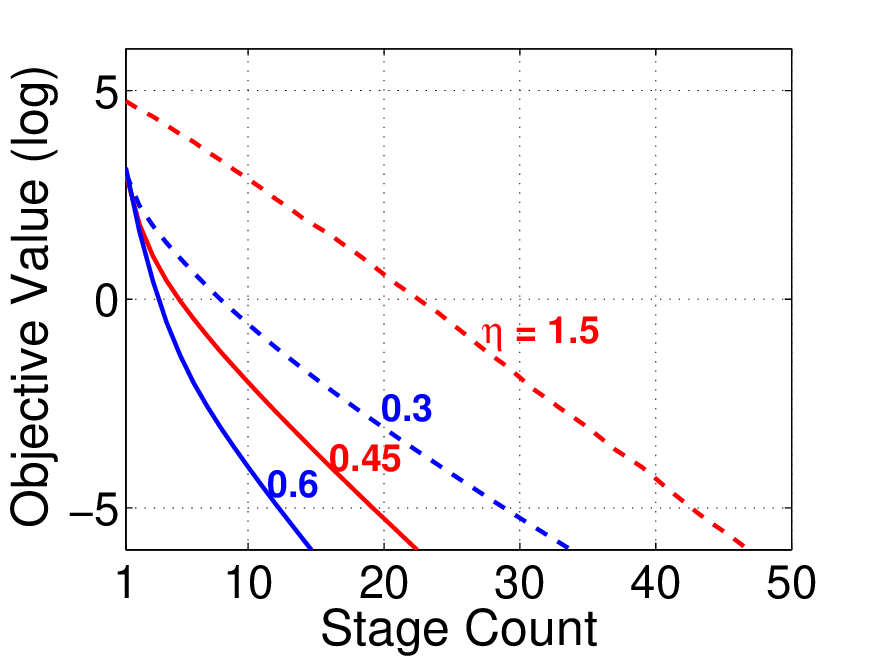}
\caption{{\bf Convergence of HT-SVRG with different parameters.} We have $100$ measurements for the $256$-dimensional signal where only $4$ elements are non-zero. The standard setting is $k = 36$, $m = 300$ and $\eta = 0.3$. {\bf Left:} If the sparsity parameter $k$ is not large enough, HT-SVRG will not recover the signal. {\bf Middle:} A small $m$ leads to a frequent full gradient evaluation and hence slow convergence. {\bf Right:} We observe divergence when $\eta \geq 3$.}
\label{fig:parameter}
\end{figure*}

The left panel of Figure~\ref{fig:parameter} verifies the condition that $k$ has to be larger than $K$, while the second panel shows the update frequency $m$ can be reasonably small in the price of a slow convergence rate. Finally, the empirical study demonstrates that our heuristic choice $\eta = 0.3$ works well, and when $\eta > 3$, the objective value exceeds $10^{120}$ within 3 stages (which cannot be depicted in the figure). For very small step sizes, we plot the convergence curve by gradually enlarging the update frequency $m$ in Figure~\ref{fig:small eta}. The empirical results agree with Theorem~\ref{thm:general} that for any $0 < \eta < 1/(4L)$, HT-SVRG converges as soon as $m$ is large enough.

\begin{figure*}[h]
\centering
\includegraphics[width=0.32\linewidth]{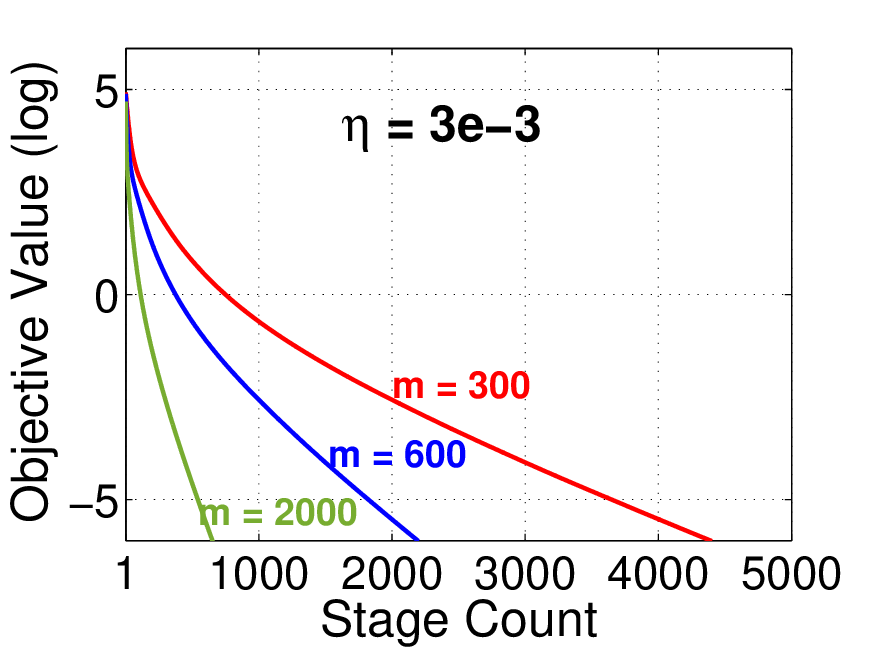}
\includegraphics[width=0.32\linewidth]{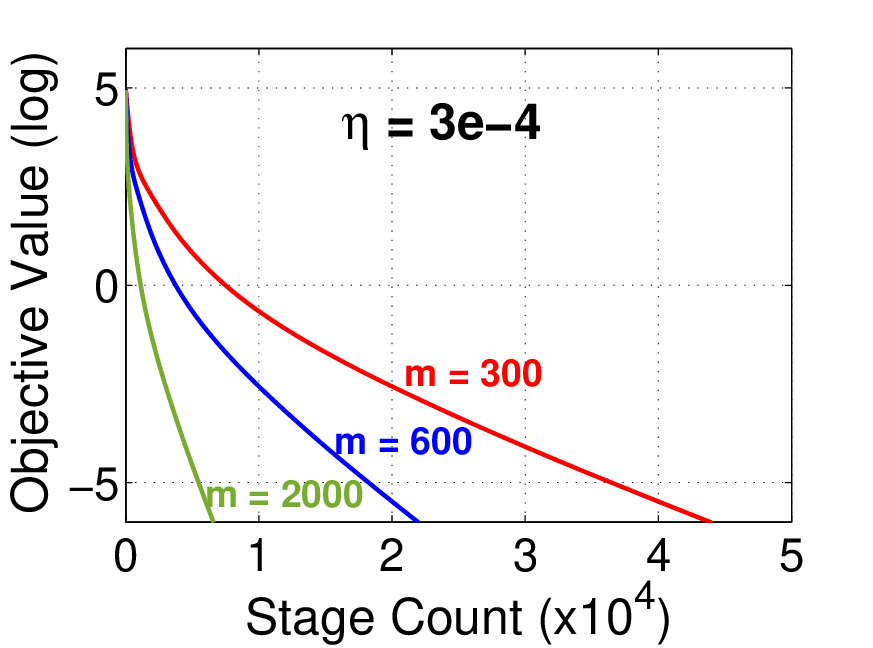}
\includegraphics[width=0.32\linewidth]{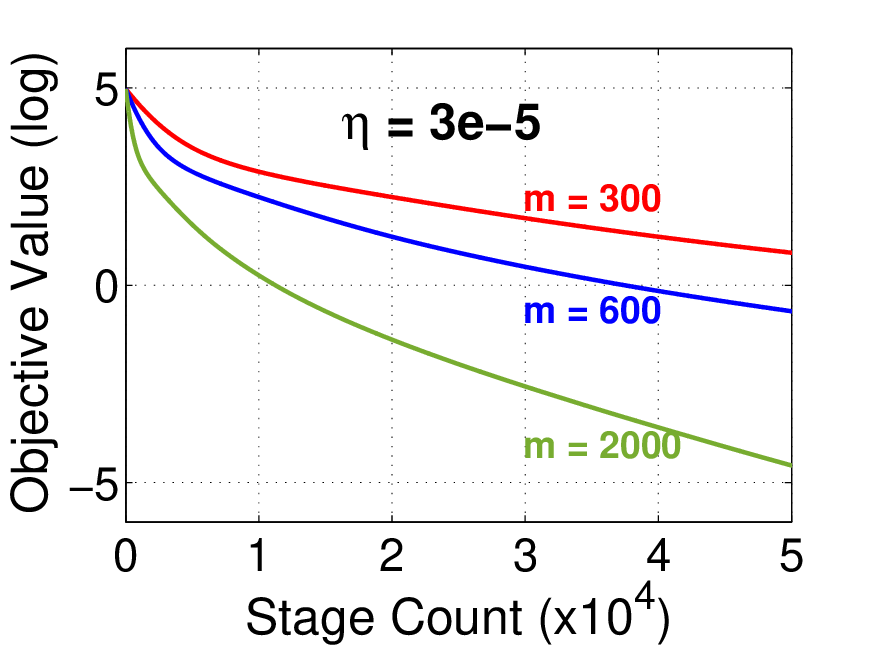}
\caption{{\bf Convergence behavior under small step size.} We observe that as long as we pick a sufficiently large value for $m$, HT-SVRG always converges. This is not surprising since our theorem guarantees for any $\eta < 1/(4L)$, HT-SVRG will converge if $m$ is large enough. Also note that the geometric convergence rate is observed after certain iterations, e.g., for $\eta = 3\times 10^{-5}$, the log(error) decreases linearly after 20 thousands iterations.}
\label{fig:small eta}
\end{figure*}

\subsection{Classification}
In addition to the application of sparse recovery, we illustrated that HT-SVRG can deal with binary classification by minimizing the sparse logistic regression problem~\eqref{eq:logistic}. Here, we study the performance on a realistic image dataset MNIST\footnote{\url{http://yann.lecun.com/exdb/mnist/}}, consisting of 60 thousands training samples and 10 thousands samples for testing. There is one digit on each image of size 28-by-28, hence totally 10 classes. Some of the images are shown in Figure~\ref{fig:mnist_data}.

\begin{figure*}[h]
\centering
\includegraphics[width=0.8\linewidth]{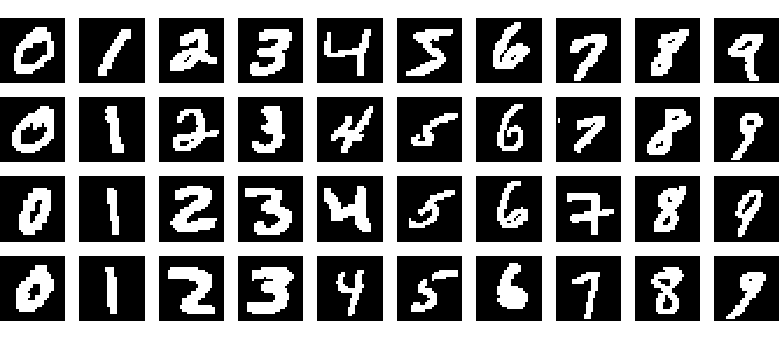}
\caption{{\bf Sample images in the MNIST database.}}
\label{fig:mnist_data}
\end{figure*}

The update frequency $m$ is fixed as $m = 3n$. We compute the heuristic step size $\eta$ as in the previous section, i.e., $\eta = 2/\mathrm{svds}(\bA \bA\trans) \approx 10^{-3}$. Since for the real-world dataset, the true sparsity is actually unknown, we tune the sparsity parameter $k$ and study the performance of the algorithm.

First, we visualize five pair-wise models learned by HT-SVRG in Figure~\ref{fig:mnist_partial_model}, where each row is associated with a binary classification task indicated by the two digits at the leading of the row, and the subsequent red-blue figures are used to illustrate the learned models under different sparsity parameter. For example, the third colorful figure depicted on the second row corresponds to recognizing a digit is ``1'' or ``7'' with the sparsity $k = 30$. In particular, for each pair, we label the small digit as positive and the large one as negative, and the blue and red pixels are the weights with positive and negative values respectively. Apparently, the models we learned are discriminative.

\begin{figure*}[h]
\centering
\includegraphics[width=0.8\linewidth]{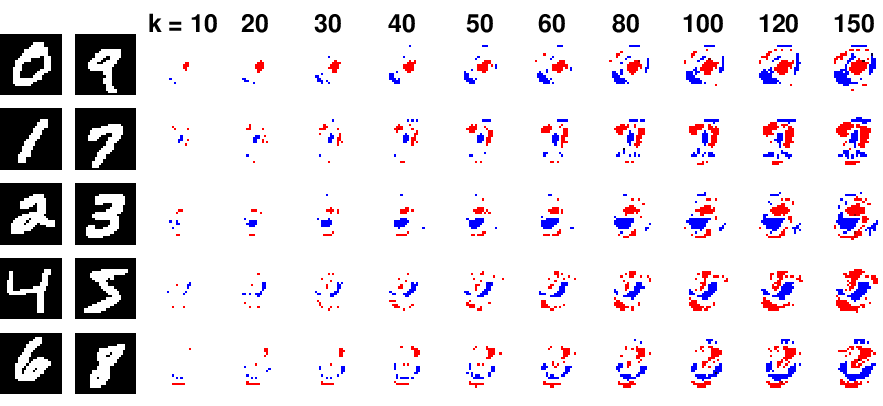}
\caption{{\bf Visualization of the models.} We visualize 5 models learned by HT-SVRG under different choices of sparsity shown on the top of each column. Note that the feature dimension is 784. From the top row to the bottom row, we illustrate the models of ``0 vs 9'', ``1 vs 7'', ``2 vs 3'', ``4 vs 5'' and ``6 vs 8'', where for each pair, we label the small digit as positive and the large one as negative. The red color represents negative weights while the blue pixels correspond with positive weights.}
\label{fig:mnist_partial_model}
\end{figure*}

We also quantitatively show the convergence and prediction accuracy curves in Figure~\ref{fig:mnist_converge_acc}. Note that here, the $y$-axis is the objective value $F(\xtilde^s)$ rather than $\log(F(\xtilde^s) - F(\xopt))$, due to the fact that computing the exact optimum of~\eqref{eq:logistic} is NP-hard. Generally speaking, HT-SVRG converges quite fast and usually attains the minimum of objective value within 20 stages. It is not surprising to see that choosing a large quantity for the sparsity leads to a better (lower) objective value. However, in practice a small assignment for the sparsity, e.g., $k=70$ facilitates an efficient computation while still suffices to ensure fast convergence and accurate prediction.

\begin{figure*}[h]
\centering
\includegraphics[width=0.32\linewidth]{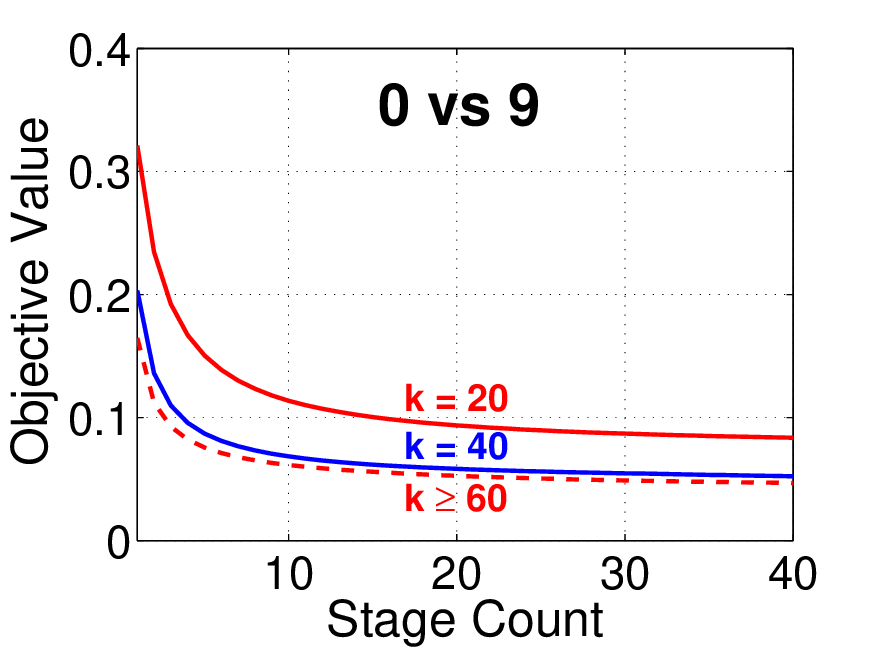}
\includegraphics[width=0.32\linewidth]{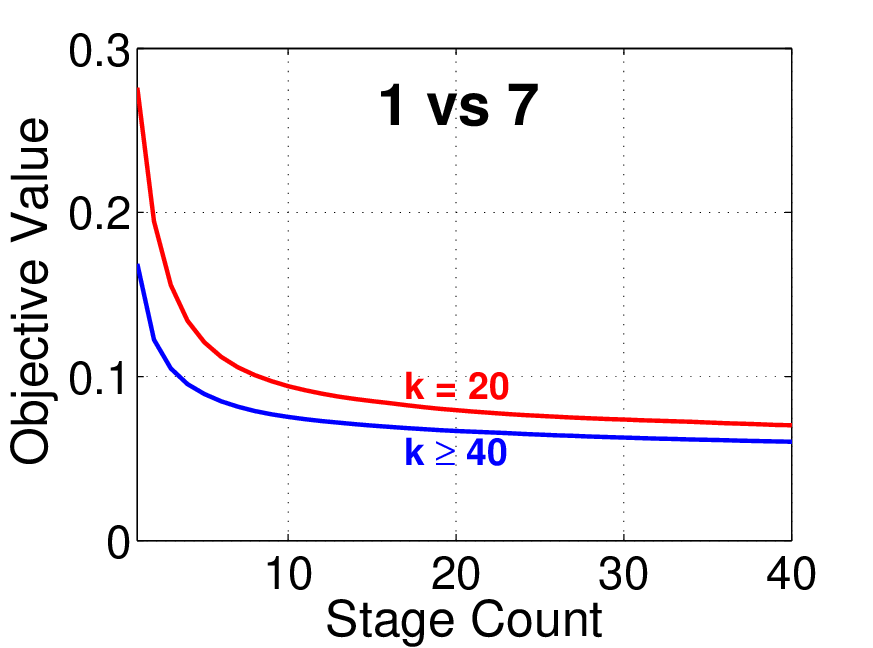}
\includegraphics[width=0.32\linewidth]{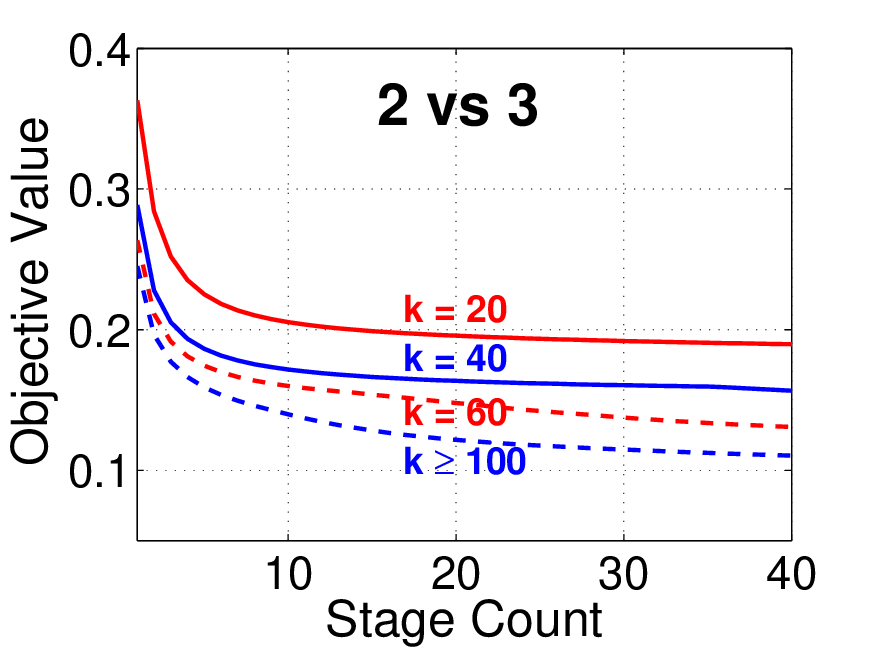}

\includegraphics[width=0.32\linewidth]{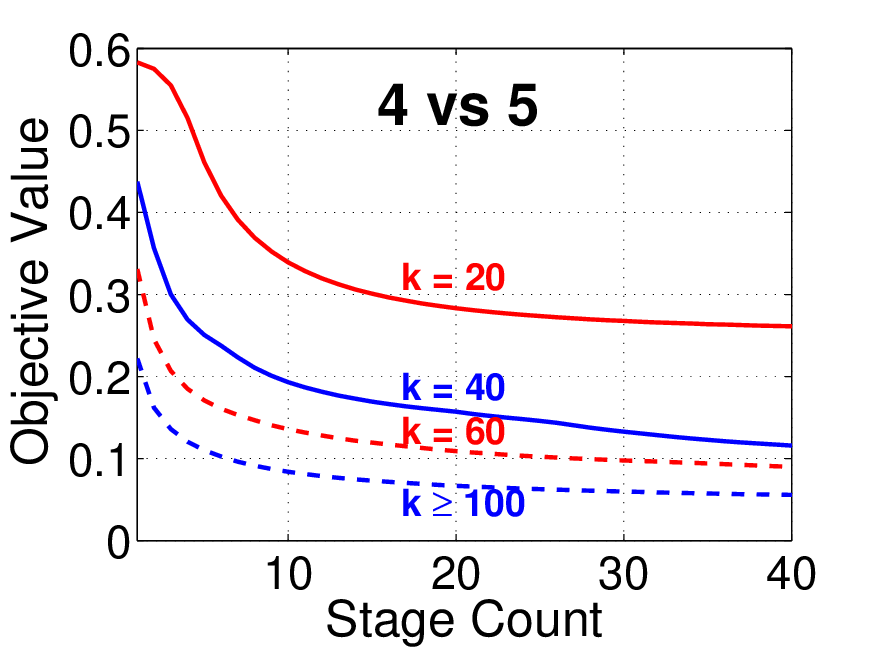}
\includegraphics[width=0.32\linewidth]{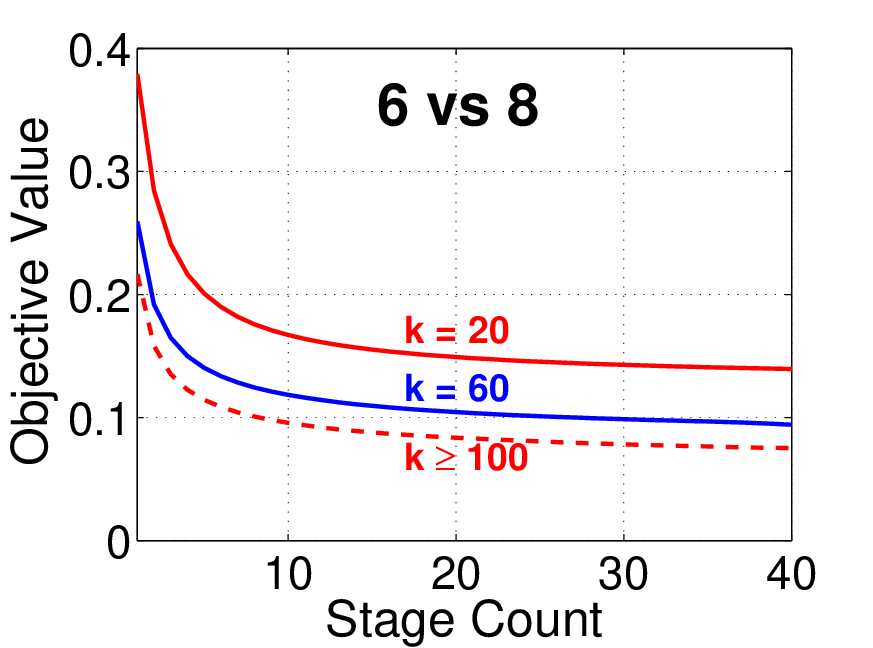}
\includegraphics[width=0.32\linewidth]{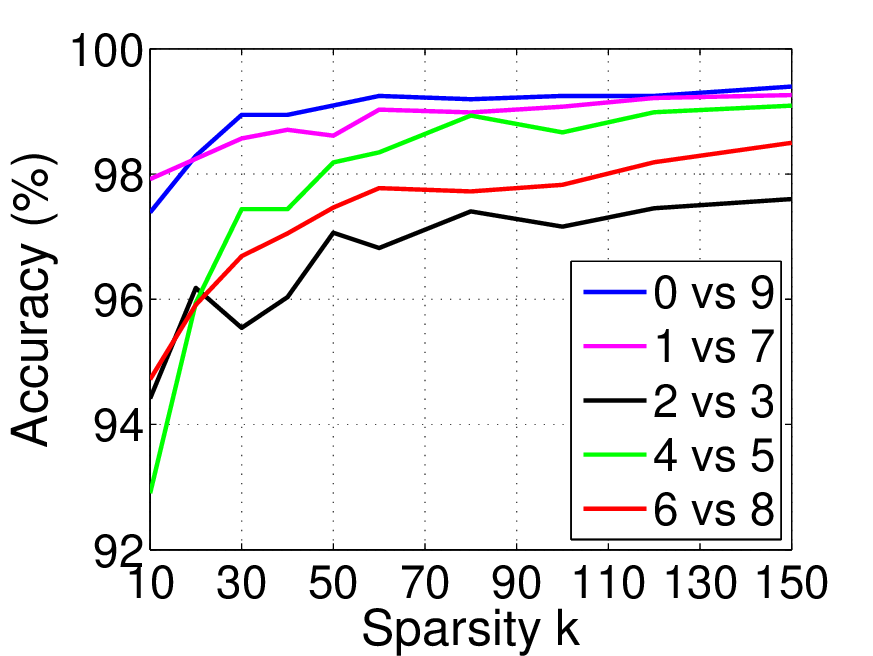}
\caption{{\bf Quantitative results on convergence and accuracy.} The first 5 figures demonstrate the convergence behavior of HT-SVRG for each binary classification task, where curves with different colors represent the objective value against number of stages under different sparsity $k$. Generally speaking, HT-SVRG converges within 20 stages which is a very fast rate. The last figure reflects the classification accuracy against the sparsity for all 5 classification tasks, where we find that for a moderate choice, e.g., $k=70$, it already guarantees an accurate prediction (we recall the dimension is 784).}
\label{fig:mnist_converge_acc}
\end{figure*}

\section{Conclusion and Open Problems}\label{sec:con}
In this paper, we have provided a tight bound on the deviation resulting from the hard thresholding operator, which underlies a vast volume of algorithms developed for sparsity-constrained problems. Our derived bound is universal over all choices of parameters and we have proved that it cannot be improved without further information on the signals. We have discussed the implications of our result to the community of compressed sensing and machine learning, and have demonstrated that the theoretical results of a number of popular algorithms in the literature can be advanced. In addition, we have devised a novel algorithm which tackles the problem of sparse learning in large-scale setting. We have elaborated that our algorithm is guaranteed to produce global optimal solution for prevalent statistical models only when it is equipped with the tight bound, hence justifying that the conventional bound is not applicable in the challenging scenario.

There are several interesting open problems. The first question to ask is whether one can establish sharp RIP condition or sharp phase transition for hard thresholding based algorithms such as IHT and CoSaMP with the tight bound. Moreover, compared to the hard thresholded SGD method~\cite{nguyen2014linear}, HT-SVRG admits a vanishing optimization error. This poses a question of whether we are able to provably show the necessity of variance reduction for such a sparsity-constrained problem.

\subsection*{Acknowledgments}
We would like to thank Jing Wang for insightful discussion since the early stage of the work. We also thank Martin Slawski for helpful discussion on the statistical precision of the problem, and thank Jian Wang for bringing the paper~\cite{nguyen2014linear} into our attention. We appreciate Huan Xu's high level comments on the work. Finally, we thank the anonymous reviewers for a careful check on our proof and for the encouraging comments. The work was partially funded by NSF-Bigdata-1419210 and NSF-III-1360971.

\appendix
\section{Technical Lemmas}
We present some useful lemmas that will be invoked by subsequent analysis. 
The following is a characterization of the co-coercivity of the objective function $F(\bx)$. A similar result was obtained in~\cite{nguyen2014linear} but we present a refined analysis which is essential for our purpose.
\begin{lemma}\label{lem:co}
For a given support set $\Omega$, assume that the continuous function $F(\bx)$ is $L_{\abs{\Omega}}$-RSS and is $\alpha_K$-RSC for some sparsity level $K$. Then, for all vectors $\bx$ and $\bx'$ with $\abs{ \supp{\bx - \bx'} \cup \Omega } \leq K$,
\begin{equation*}
\twonorm{ \nabla_{\Omega} F(\bx') - \nabla_{\Omega} F(\bx) }^2 \leq 2 L_{\abs{\Omega}} \big( F(\bx') - F(\bx) - \inner{\nabla F(\bx)}{\bx' - \bx} \big).
\end{equation*}
\end{lemma}
\begin{proof}
We define an auxiliary function
\begin{equation*}
G(\bw) \defeq F(\bw) - \inner{\nabla F(\bx)}{\bw}.
\end{equation*}
For all vectors $\bw$ and $\bw'$, we have
\begin{equation*}
\twonorm{\nabla G(\bw) - \nabla G(\bw')} = \twonorm{\nabla F(\bw) - \nabla F(\bw')}  \leq L_{\abs{\supp{\bw - \bw'}}} \twonorm{\bw - \bw'},
\end{equation*}
which is equivalent to
\begin{equation}\label{eq:tmp:1}
G(\bw) - G(\bw') - \inner{\nabla G(\bw')}{\bw - \bw'} \leq \frac{L_{r}}{2} \twonorm{\bw - \bw'}^2,
\end{equation}
where $r \defeq \abs{\supp{\bw - \bw'}}$. On the other hand, due to the RSC property of $F(\bx)$, we obtain
\begin{equation*}
G(\bw) - G(\bx) = F(\bw) - F(\bx) - \inner{\nabla F(\bx)}{\bw - \bx} \geq \frac{\alpha_{\abs{\supp{\bw - \bx}}}}{2} \twonorm{\bw - \bx}^2 \geq 0,
\end{equation*}
provided that $\abs{\supp{\bw - \bx}} \leq K$. For the given support set $\Omega$, we pick $\bw = \bx' - \frac{1}{L_{\abs{\Omega}}} \nabla_{\Omega} G(\bx')$. Clearly, for such a choice of $\bw$, we have ${\supp{\bw - \bx}} = \supp{\bx - \bx'} \cup \Omega$. Hence, by assuming that $\abs{\supp{\bx - \bx'} \cup \Omega }$ is not larger than $K$, we get
\begin{align*}
G(\bx) \leq&\ G\(\bx' - \frac{1}{L_{\abs{\Omega}}} \nabla_{\Omega} G(\bx')\)\\
{\leq}&\ G(\bx') + \inner{\nabla G(\bx')}{- \frac{1}{L_{\abs{\Omega}}} \nabla_{\Omega} G(\bx')}  + \frac{1}{2 L_{\abs{\Omega}}} \twonorm{\nabla_{\Omega} G(\bx')}^2\\
=&\ G(\bx') - \frac{1}{2 L_{\abs{\Omega}}} \twonorm{\nabla_{\Omega} G(\bx')}^2,
\end{align*}
where the second inequality follows from~\eqref{eq:tmp:1}. Now expanding $\nabla_{\Omega} G(\bx')$ and rearranging the terms gives the desired result.
\end{proof}

\begin{lemma}\label{lem:vt norm}
Consider the HT-SVRG algorithm for a fixed stage $s$. Let $\xhat$ be the target sparse vector. Let $\Omega$ be a support set such that $\supp{\bx^{t-1}} \cup \supp{\xtilde} \cup \supp{\xhat} \subseteq \Omega$. Put $r = \abs{\Omega}$. Assume \ref{as:rss}. For all $1 \leq t \leq m$ , denote $\bv^t = \nabla f_{i_t}(\bx^{t-1}) - \nabla f_{i_t}(\xtilde) + \mutilde$. Then we have the following:
\begin{align*}
{\E}_{i_t | \bx^{t-1}} \Big[\twonorm{\PO{\bv^t}}^2 \Big] \leq&\ 4 L_r \left[ F(\bx^{t-1}) - F(\xhat)\right] + 4 L_r \left[ F(\xtilde) - F(\xhat) \right] \notag\\
&\ -4L_r \inner{\nabla F(\xhat)}{\bx^{t-1} + \xtilde - 2 \xhat} + 4 \twonorm{\PO{\nabla F(\xhat)}}^2.
\end{align*}
\end{lemma}
\begin{proof}
We have
\begin{align*}
\twonorm{\PO{\bv^t}}^2 =&\ \twonorm{\PO{ \nabla f_{i_t}(\bx^{t-1}) - \nabla f_{i_t}(\xtilde) + \mutilde }}^2\\
\leq&\ 2 \twonorm{ \PO{\nabla f_{i_t}(\bx^{t-1}) - \nabla f_{i_t}(\xhat) } }^2 + 2 \twonorm{ \PO{\nabla f_{i_t}(\xtilde) - \nabla f_{i_t}(\xhat) -\mutilde} }^2\\
=&\ 2 \twonorm{ \PO{\nabla f_{i_t}(\bx^{t-1}) - \nabla f_{i_t}(\xhat) } }^2 +  2 \twonorm{ \PO{\nabla f_{i_t}(\xtilde) - \nabla f_{i_t}(\xhat)} }^2 \notag\\
&\ + 2 \twonorm{\PO{\mutilde}}^2 - 4 \inner{\PO{\nabla f_{i_t}(\xtilde) - \nabla f_{i_t}(\xhat)} }{ \PO{\mutilde}}\\
\stackrel{\xi_1}{=}&\ 2 \twonorm{ \PO{\nabla f_{i_t}(\bx^{t-1}) - \nabla f_{i_t}(\xhat) } }^2 +  2 \twonorm{ \PO{\nabla f_{i_t}(\xtilde) - \nabla f_{i_t}(\xhat)} }^2 \notag\\
&\ +  2 \twonorm{\PO{\mutilde}}^2 - 4 \inner{{\nabla f_{i_t}(\xtilde) - \nabla f_{i_t}(\xhat)} }{ \PO{\mutilde}}\\
\stackrel{\xi_2}{\leq}&\ 4L_r\left[ f_{i_t}(\bx^{t-1}) -  f_{i_t}(\xhat) - \inner{\nabla f_{i_t}(\xhat)}{\bx^{t-1} - \xhat} \right] \notag\\
&\ + 4L_r\left[ f_{i_t}(\xtilde) -  f_{i_t}(\xhat) - \inner{\nabla f_{i_t}(\xhat)}{\xtilde - \xhat} \right] \notag\\
&\ +  2 \twonorm{\PO{\mutilde}}^2 - 4 \inner{{\nabla f_{i_t}(\xtilde) - \nabla f_{i_t}(\xhat)} }{ \PO{\mutilde}},
\end{align*}
where $\xi_1$ is by algebra, $\xi_2$ applies Lemma~\ref{lem:co} and the fact that $\abs{\Omega} = r$.

Taking the conditional expectation, we obtain the following:
\begin{align*}
&\ {\E}_{i_t | \bx^{t-1}} \Big[ \twonorm{\PO{\bv^t}}^2\Big]\\
 \leq&\ 4 L_r \left[ F(\bx^{t-1}) - F(\xhat)\right] + 4 L_r \left[ F(\xtilde) - F(\xhat) \right]\notag\\
&\ -4L_r \inner{\nabla F(\xhat)}{\bx^{t-1} + \xtilde - 2\xhat} + 2\inner{2\PO{\nabla F(\xhat)} - \PO{\mutilde} }{\PO{\mutilde}}\\
=&\ 4 L_r \left[ F(\bx^{t-1}) - F(\xhat)\right] + 4 L_r \left[ F(\xtilde) - F(\xhat) \right]\notag\\
&\ -4L_r \inner{\nabla F(\xhat)}{\bx^{t-1} + \xtilde - 2\xhat} + \twonorm{2 \PO{\nabla F(\xhat)}}^2 \notag\\
&\ - \twonorm{2\PO{\nabla F(\xhat)} - \PO{\mutilde}}^2 - \twonorm{\PO{\mutilde}}^2\\
\leq&\ 4 L_r \left[ F(\bx^{t-1}) - F(\xhat)\right] + 4 L_r \left[ F(\xtilde) - F(\xhat) \right]\notag\\
&\ -4L_r \inner{\nabla F(\xhat)}{\bx^{t-1} + \xtilde - 2\xhat} + 4 \twonorm{\PO{\nabla F(\xhat)}}^2.
\end{align*}
The proof is complete.
\end{proof}

\begin{corollary}\label{cor:vt norm clean}
Assume the same conditions as in Lemma~\ref{lem:vt norm}. If $\nabla F(\xhat) = 0$, we have
\begin{equation*}
{\E}_{i_t | \bx^{t-1}} \Big[ \twonorm{\PO{\bv^t}}^2 \Big] \leq 4L_r \left[ F(\bx^{t-1})+ F(\xtilde) - 2F(\xhat) \right].
\end{equation*}
\end{corollary}

\section{Proofs for Section~\ref{sec:key}}\label{sec:app:proofkey}
\subsection{Proof of Theorem~\ref{thm:key}}
\begin{proof}
The result is true for the trivial case that $\bb$ is a zero vector. In the following, we assume that $\bb$ is not a zero vector. Denote
\begin{equation*}
\bw \defeq \Hk{\bb}.
\end{equation*}
Let $\Omega$ be the support set of $\bw$ and let $\overline{\Omega}$ be its complement. We immediately have $\PO{\bb} = \bw$.

Let $\Omega'$ be the support set of $\bx$. For the sake of simplicity, let us split the vector $\bb$ as follows:
\begin{align*}
\bb_1 = \mathcal{P}_{\Omega \backslash \Omega'}\( \bb \),\quad& \bb_2 = \mathcal{P}_{\Omega \cap \Omega'}\( \bb \),\\
\bb_3 = \mathcal{P}_{\overline{\Omega} \backslash \Omega'}\( \bb \),\quad& \bb_4 = \mathcal{P}_{\overline{\Omega} \cap \Omega'}\( \bb \).
\end{align*}
Likewise, we denote
\begin{align*}
&\bw_1 = \mathcal{P}_{\Omega \backslash \Omega'}\( \bw \),\quad \bw_2 = \mathcal{P}_{\Omega \cap \Omega'}\( \bw \),\quad \bw_3 = \mathcal{P}_{\overline{\Omega} \backslash \Omega'}( \bw ) = \bzero,\quad \bw_4 = \mathcal{P}_{\overline{\Omega} \cap \Omega'}\( \bw \) = \bzero,\\
&\bx_1 = \mathcal{P}_{\Omega \backslash \Omega'}\( \bx \) = \bzero,\quad \bx_2 = \mathcal{P}_{\Omega \cap \Omega'}\( \bx \),\quad \bx_3 = \mathcal{P}_{\overline{\Omega} \backslash \Omega'}( \bx ) = \bzero,\quad \bx_4 = \mathcal{P}_{\overline{\Omega} \cap \Omega'}\( \bx \).
\end{align*}
Due to the hard thresholding, we have
\begin{equation*}
\bw_1 = \bb_1,\quad \bw_2 = \bb_2.
\end{equation*}
In this way, by simple algebra we have
\begin{align*}
\twonorm{\bw - \bx}^2 &=\twonorm{\bb_1}^2 + \twonorm{\bb_2 - \bx_2}^2 + \twonorm{\bx_4}^2,\\
\twonorm{\bb - \bx}^2 &= \twonorm{\bb_1}^2 + \twonorm{\bb_2 - \bx_2}^2 + \twonorm{\bb_3}^2 + \twonorm{\bb_4 - \bx_4}^2.
\end{align*}

Our goal is to estimate the maximum of $\twonorm{\bw - \bx}^2 / \twonorm{\bb - \bx}^2$. It is easy to show that when attaining the maximum value, $\twonorm{\bb_3}$ must be zero since otherwise one may decrease this term to make the objective larger. Hence, maximizing $\twonorm{\bw - \bx}^2 / \twonorm{\bb -\bx}^2$ amounts to estimating the upper bound of the following over all choices of $\bx$ and $\bb$:
\begin{equation}\label{eq:t}
\gamma \defeq  \frac{\twonorm{\bb_1}^2 + \twonorm{\bb_2 - \bx_2}^2 + \twonorm{\bx_4}^2}{\twonorm{\bb_1}^2 + \twonorm{\bb_2 - \bx_2}^2  + \twonorm{\bb_4 - \bx_4}^2}.
\end{equation}

Firstly, we consider the case of $\twonorm{\bb_1} = 0$, which means $\Omega = \Omega'$ implying $\gamma=1$. In the following, we consider $\twonorm{\bb_1} \neq 0$. In particular, we consider $\gamma > 1$ since we are interested in the maximum value of $\gamma$.

Arranging \eqref{eq:t} we obtain
\begin{equation}\label{eq:G(x)=0}
(\gamma-1) \twonorm{\bb_2 - \bx_2}^2 + \gamma \twonorm{\bb_4 - \bx_4}^2 - \twonorm{\bx_4}^2 + (\gamma-1) \twonorm{\bb_1}^2 = 0.
\end{equation}
Let us fix $\bb$ and define the function
\begin{equation*}
G(\bx_2, \bx_4) = (\gamma-1) \twonorm{\bb_2 - \bx_2}^2 + \gamma \twonorm{\bb_4 - \bx_4}^2 - \twonorm{\bx_4}^2 + (\gamma-1) \twonorm{\bb_1}^2.
\end{equation*}
Thus, \eqref{eq:G(x)=0} indicates that $G(\bx_2, \bx_4)$ can attain the objective value of zero. Note that $G(\bx_2, \bx_4)$ is a quadratic function and its gradient and Hessian matrix can be computed as follows:
\begin{align*}
\frac{\partial}{\partial \bx_2} G(\bx_2, \bx_4) &= 2(\gamma-1)(\bx_2 - \bb_2),\\
\frac{\partial}{\partial \bx_4} G(\bx_2, \bx_4) &= 2\gamma(\bx_4 - \bb_4) - 2\bx_4,\\
\nabla^2 G(\bx_2, \bx_4) &= 2(\gamma-1) \bI,
\end{align*}
where $\bI$ is the identity matrix. Since the Hessian matrix is positive definite, $G(\bx_2, \bx_4)$ attains the global minimum at the stationary point, which is given by
\begin{equation*}
\bx_2^* = \bb_2,\quad \bx_4^* = \frac{\gamma}{\gamma-1} \bb_4,
\end{equation*}
resulting in the minimum objective value
\begin{equation*}
G(\bx_2^*, \bx_4^*) = \frac{\gamma}{1-\gamma} \twonorm{\bb_4}^2 + (\gamma-1) \twonorm{\bb_1}^2.
\end{equation*}
In order to guarantee the feasible set of \eqref{eq:G(x)=0} is non-empty, we require that
\begin{equation*}
G(\bx_2^*, \bx_4^*) \leq 0,
\end{equation*}
implying
\begin{equation*}
\twonorm{\bb_1}^2 \gamma^2 - (2\twonorm{\bb_1}^2 + \twonorm{\bb_4}^2) \gamma + \twonorm{\bb_1}^2 \leq 0.
\end{equation*}
Solving the above inequality with respect to $\gamma$, we obtain
\begin{equation}\label{eq:t est1}
\gamma \leq 1 + \frac{\twonorm{\bb_4}^2 + \sqrt{\(4\twonorm{\bb_1}^2+\twonorm{\bb_4}^2\) \twonorm{\bb_4}^2}}{2\twonorm{\bb_1}^2}.
\end{equation}
To derive an upper bound that is uniform over the choice of $\bb$, we recall that $\bb_1$ contains the largest absolute elements of $\bb$ while $\bb_4$ has smaller values. In particular, the averaged value of $\bb_4$ is no greater than that of $\bb_1$ in magnitude, i.e.,
\begin{equation*}
\frac{\twonorm{\bb_4}^2}{\zeronorm{\bb_4}} \leq \frac{\twonorm{\bb_1}^2}{\zeronorm{\bb_1}}.
\end{equation*}
Note that $\zeronorm{\bb_1} = k - \zeronorm{\bb_2} = k - (K - \zeronorm{\bb_4})$. Hence, combining with the fact that $0 \leq \zeronorm{\bb_4} \leq \min\{K, d-k\}$ and optimizing over $\zeronorm{\bb_4}$ gives
\begin{equation*}
\twonorm{\bb_4}^2 \leq \frac{\min\{K, d-k\}}{k - K + \min\{K, d-k\}} \twonorm{\bb_1}^2.
\end{equation*}
Plugging back to \eqref{eq:t est1}, we finally obtain
\begin{equation*}
\gamma \leq 1 + \frac{\rho + \sqrt{\(4 + \rho \) \rho } }{2},\quad \rho =  \frac{\min\{K, d-k\}}{k - K + \min\{K, d-k\}}.
\end{equation*}
The proof is complete.
\end{proof}

\section{Proofs for Section~\ref{sec:imp}}\label{sec:app:proofimp}
\subsection{Proof of Theorem~\ref{thm:iht}}

We follow the proof pipeline of~\cite{blumensath2009iterative} and only remark the difference of our proof and theirs, i.e., where Theorem~\ref{thm:key} applies. In case of possible confusion due to notation, we follow the symbols in Blumensath and Davies. One may refer to that article for a complete proof.

The first difference occurs in Eq.~(22) of~\cite{blumensath2009iterative}, where they reached
\begin{equation*}
\textrm{(Old)}\quad \twonorm{\bx^s - \bx^{[n+1]}} \leq 2 \twonorm{\bx^s_{B^{n+1}} - \ba^{[n+1]}_{B^{n+1}}},
\end{equation*}
while Theorem~\ref{thm:key} gives
\begin{equation*}
\textrm{(New)}\quad \twonorm{\bx^s - \bx^{[n+1]}} \leq \sqrt{\nu} \twonorm{\bx^s_{B^{n+1}} - \ba^{[n+1]}_{B^{n+1}}}.
\end{equation*}
Combining this new inequality and Eq.~(23) therein, we obtain
\begin{equation*}
\twonorm{\bx^s - \bx^{[n+1]}} \leq \sqrt{\nu} \twonorm{(\bI - \bPhi\trans_{B^{n+1}} \bPhi_{B^{n+1}} ) \br^{[n]}_{B^{n+1}}} + \sqrt{\nu} \twonorm{(\bPhi\trans_{B^{n+1}} \bPhi_{B^{n+1} \backslash B^{n+1}} ) \br^{[n]}_{B^{n+1} \backslash B^{n+1}}}.
\end{equation*}
By noting the fact that $\abs{B^n \cup B^{n+1}} \leq 2s + s^*$ where $s^*$ denotes the sparsity of the global optimum and following their reasoning of Eq.~(24) and (25), we have a new bound for Eq.~(26):
\begin{equation*}
\textrm{(New)}\quad \twonorm{\br^{[n+1]}} \leq \sqrt{2\nu} \delta_{2s+s^*} \twonorm{\br^{[n]}} + \sqrt{(1+\delta_{s+s^*})\nu} \twonorm{\be}.
\end{equation*}
Now our result follows by setting the coefficient of $\twonorm{\br^{[n]}}$ to $0.5$. Note that specifying $\nu=4$ gives the result of~\cite{blumensath2009iterative}.

\subsection{Proof of Theorem~\ref{thm:cosamp}}

We follow the proof technique of Theorem 6.27 in~\cite{foucart2013mathematical} which gives the best known RIP condition for the CoSaMP algorithm to date. Since most of the reasoning is similar, we only point out the difference of our proof and theirs, i.e., where Theorem~\ref{thm:key} applies. In case of confusion by notation, we follow the symbols used in~\cite{foucart2013mathematical}. The reader may refer to that book for a complete proof.

The first difference is in Eq.~(6.49) of~\cite{foucart2013mathematical}. Note that to derive this inequality, Foucart and Rauhut invoked the conventional bound~\eqref{eq:old}, which gives
\begin{equation*}
	\textrm{(Old)}\quad \twonorm{\bx_S - \bx^{n+1}}^2 \leq \twonorm{(\bx_S - \bu^{n+1})_{\overline{U^{n+1}}}}^2 + 4 \twonorm{(\bx_S - \bu^{n+1})_{U^{n+1}}}^2,
\end{equation*}
while utilizing Theorem~\ref{thm:key} gives
\begin{equation*}
	\textrm{(New)}\quad \twonorm{\bx_S - \bx^{n+1}}^2 \leq \twonorm{(\bx_S - \bu^{n+1})_{\overline{U^{n+1}}}}^2 + \nu \twonorm{(\bx_S - \bu^{n+1})_{U^{n+1}}}^2.
\end{equation*}
Combining this new inequality with Eq.~(6.50) and Eq.~(6.51) therein, we obtain
\begin{align*}
	\twonorm{\bx_S - \bx^{n+1}} \leq&\ \sqrt{2}\delta_{3s+s^*} \sqrt{\frac{1+(\nu-1)\delta_{3s+s^*}^2}{1-\delta_{3s+s^*}^2}} \twonorm{\bx^n - \bx_S} \\
	& + \sqrt{2}\delta_{3s+s^*} \sqrt{\frac{1+(\nu-1)\delta_{3s+s^*}^2}{1-\delta_{3s+s^*}^2}} \twonorm{(\bA^*\be')_{(S\cup S^n)\Delta T^{n+1}}}\\
	& + \frac{2}{1 - \delta_{3s+s^*}} \twonorm{(\bA^*\be')_{U^{n+1}}},
\end{align*}
where $s^*$ denotes the sparsity of the optimum. Our new bound follows by setting the coefficient of $\twonorm{\bx^n - \bx_S}$ to $0.5$ and solving the resultant equation. Note that setting $\nu=4$ gives the old bound of Foucart and Rauhut.

\section{Proofs for Section~\ref{sec:alg}}\label{sec:app:proofalg}
\subsection{Proof of Theorem~\ref{thm:general}}\label{sec:app:proofgeneral}
\begin{proof}
Fix a stage $s$. Let us denote
\begin{equation*}
\bv^t = \nabla f_{i_t}(\bx^{t-1}) - \nabla f_{i_t}(\xtilde) + \mutilde,
\end{equation*}
so that
\begin{equation*}
\bb^t = \bx^{t-1} - \eta \bv^t.
\end{equation*}
By specifying $\Omega = \supp{\bx^{t-1}} \cup \supp{\bx^t} \cup \supp{\xtilde} \cup \supp{\xhat}$, it follows that
\begin{equation*}
\br^t = \Hk{\bb^t} = \Hk{\PO{\bb^t}}.
\end{equation*}
Thus, the Euclidean distance of $\bx^t$ and $\xhat$ can be bounded as follows:
\begin{equation}\label{eq:app1}
\twonorm{\bx^t - \xhat}^2 \leq \twonorm{\br^t - \xhat}^2 = \twonorm{\Hk{\PO{\bb^t}} - \xhat }^2 \leq \nu \twonorm{\PO{\bb^t} - \xhat}^2,
\end{equation}
where the first inequality holds because $\bx^t = \Pi_{\omega}(\br^t)$ and $\twonorm{\xhat} \leq \omega$. We also have
\begin{align*}
\twonorm{\PO{\bb^t} - \xhat}^2 &= \twonorm{\bx^{t-1} - \xhat - \eta \PO{ \bv^t}}^2\\
&= \twonorm{\bx^{t-1} - \xhat}^2 + \eta^2 \twonorm{\PO{\bv^t}}^2 - 2\eta \inner{\bx^{t-1} - \xhat}{\bv^t},
\end{align*}
where the second equality uses the fact that $\inner{\bx^{t-1} - \xhat}{\PO{\bv^t}} = \inner{\bx^{t-1} - \xhat}{{\bv^t}}$. The first term will be preserved for mathematical induction. The third term is easy to manipulate thanks to the unbiasedness of $\bv^t$. For the second term, we use Lemma~\ref{lem:vt norm} to bound it. Put them together, conditioning on $\bx^{t-1}$ and taking the expectation over $i_t$ for~\eqref{eq:app1}, we have
\begin{align*}
&\ {\E}_{i_t|\bx^{t-1}} \Big[ \twonorm{\bx^t - \xhat}^2 \Big]\notag\\
\stackrel{\xi_1}{\leq}&\ \nu \twonorm{\bx^{t-1} - \xhat}^2 + 4\nu \eta^2 L \left[ F(\bx^{t-1}) - F(\xhat) + F(\xtilde) - F(\xhat) \right] - 2 \nu \eta \inner{\bx^{t-1} - \xhat}{\nabla F(\bx^{t-1})}\notag\\
&\ -4\nu \eta^2 L \inner{\nabla F(\xhat)}{\bx^{t-1} + \xtilde - 2 \xhat} + 4\nu \eta^2 \twonorm{\PO{\nabla F(\xhat)}}^2 \\
\stackrel{\xi_2}{\leq}&\ \nu(1 - \eta \alpha) \twonorm{\bx^{t-1} - \xhat}^2 - 2\nu \eta (1 - 2\eta L) \left[ F(\bx^{t-1}) - F(\xhat) \right] + 4\nu \eta^2 L \left[ F(\xtilde) - F(\xhat) \right]\notag\\
&\ + 4\nu \eta^2 L \twonorm{\PO{\nabla F(\xhat)}} \cdot \twonorm{\bx^{t-1} + \xtilde - 2\xhat } + 4\nu \eta^2 \twonorm{\PO{\nabla F(\xhat)}}^2\\
\leq&\ \nu(1 - \eta \alpha) \twonorm{\bx^{t-1} - \xhat}^2 - 2\nu \eta (1 - 2\eta L) \left[ F(\bx^{t-1}) - F(\xhat) \right] \notag\\
&\ + 4\nu \eta^2 L \left[ F(\xtilde) - F(\xhat) \right] + 4\nu \eta^2 Q' (4 L\omega + Q')
\end{align*}
where $\xi_1$ applies Lemma~\ref{lem:vt norm}, $\xi_2$ applies Assumption~\ref{as:rsc} and we write $Q' := \twonorm{\nabla_{3k+K} F(\xhat)}$ for brevity.

Now summing over the inequalities over $t = 1, 2, \cdots, m$, conditioning on $\xtilde$ and taking the expectation with respect to $\mathcal{I}^s = \{ i_1, i_2, \cdots, i_m \}$, we have
\begin{align}\label{eq:chain}
&\ {\E}_{\mathcal{I}^s|\xtilde} \Big[ \twonorm{\bx^m - \xhat}^2 \Big] \notag\\
\leq&\ \left[\nu(1 - \eta \alpha) - 1\right] {\E}_{\mathcal{I}^s|\xtilde} \sum_{t=1}^{m} \twonorm{\bx^{t-1} - \xhat}^2 + \twonorm{\bx^0 - \xhat}^2 + 4\nu \eta^2 Q' (4L\omega + Q')m \notag\\
&\ - 2\nu \eta (1 - 2\eta L) {\E}_{\mathcal{I}^s|\xtilde} \sum_{t=1}^{m}\left[ F(\bx^{t-1}) - F(\xhat) \right] + 4\nu \eta^2 Lm \left[ F(\xtilde) - F(\xhat) \right] \notag\\
{=}&\ \left[\nu(1 - \eta \alpha) - 1\right]m {\E}_{\mathcal{I}^s, j^s|\xtilde} \twonorm{\xtilde^s - \xhat}^2 + \twonorm{\xtilde - \xhat}^2  + 4\nu \eta^2 Q' (4L\omega + Q')m\notag\\
&\ - 2\nu \eta (1 - 2\eta L)m {\E}_{\mathcal{I}^s, j^s|\xtilde} \left[ F(\xtilde^s) - F(\xhat) \right] + 4\nu \eta^2 Lm \left[ F(\xtilde) - F(\xhat) \right] \notag\\
\leq&\ \left[\nu(1 - \eta \alpha) - 1\right]m {\E}_{\mathcal{I}^s, j^s|\xtilde} \twonorm{\xtilde^s - \xhat}^2 + \(\frac{2}{\alpha} + 4\nu \eta^2 Lm \)\left[ F(\xtilde) - F(\xhat) \right] \notag\\
&\ - 2\nu \eta (1 - 2\eta L)m {\E}_{\mathcal{I}^s, j^s|\xtilde} \left[ F(\xtilde^s) - F(\xhat) \right]  + 4\nu \eta^2 Q' (4L\omega + Q')m + 2Q'\omega / \alpha,
\end{align}
where we recall that $j^s$ is the randomly chosen index used to determine $\xtilde^s$ (see Algorithm~\ref{alg:all}). The last inequality holds due to the RSC condition and $\twonorm{\bx^t} \leq \omega$. For brevity, we write
\begin{align*}
Q := 4\nu \eta^2 Q' (4L\omega + Q')m + 2Q'\omega / \alpha,\quad Q' = \twonorm{\nabla_{3k+K} F(\xhat)}.
\end{align*}
Based on \eqref{eq:chain}, we discuss two cases to examine the convergence of the algorithm.

\vskip 0.2in
\noindent{\bf Case 1.} ${\nu(1-\eta \alpha) \leq 1}$. This immediately results in
\begin{align*}
&\ {\E}_{\mathcal{I}^s|\xtilde} \Big[ \twonorm{\bx^m - \xhat}^2 \Big]\\
\leq&\ \(\frac{2}{\alpha} + 4\nu \eta^2 Lm \)\left[ F(\xtilde) - F(\xhat) \right] - 2\nu \eta (1 - 2\eta L)m\ {\E}_{\mathcal{I}^s, j^s|\xtilde} \left[ F(\xtilde^s) - F(\xhat) \right] + Q,
\end{align*}
which implies
\begin{equation*}
\nu \eta (1 - 2\eta L)m {\E}_{\mathcal{I}^s, j^s|\xtilde} \left[ F(\xtilde^s) - F(\xhat) \right] \leq \(\frac{1}{\alpha} + 2\nu \eta^2 Lm \)\left[ F(\xtilde) - F(\xhat) \right]+ \frac{Q}{2}.
\end{equation*}
Pick $\eta$ such that
\begin{equation}\label{eq:req11}
1 - 2\eta L > 0,
\end{equation}
we obtain
\begin{equation*}
{\E}_{\mathcal{I}^s, j^s|\xtilde} \left[ F(\xtilde^s) - F(\xhat) \right] \leq \( \frac{1}{\nu \eta \alpha (1 - 2\eta L)m} + \frac{2\eta L}{1 - 2\eta L} \) \left[ F(\xtilde) - F(\xhat) \right] + \frac{Q}{2\nu \eta \alpha (1 - 2\eta L)m}.
\end{equation*}
To guarantee the convergence, we must impose
\begin{equation}\label{eq:req12}
\frac{2\eta L}{1 - 2\eta L} < 1.
\end{equation}
Putting \eqref{eq:req11}, \eqref{eq:req12} and $\nu(1-\eta \alpha) \leq 1$ together gives
\begin{equation}\label{eq:case1final}
\eta < \frac{1}{4L},\quad \nu \leq \frac{1}{1-\eta \alpha}.
\end{equation}
The convergence coefficient here is
\begin{equation}\label{eq:case1beta}
\beta = \frac{1}{\nu \eta \alpha (1 - 2\eta L)m} + \frac{2\eta L}{1 - 2\eta L}.
\end{equation}
Thus, we have
\begin{equation*}
\E \left[ F(\xtilde^s) - F(\xhat) \right] \leq \beta^s \left[ F(\xtilde^0) - F(\xhat) \right] +  \frac{Q}{2\nu \eta \alpha (1 - 2\eta L)(1-\beta)m},
\end{equation*}
where the expectation is taken over $\{ \mathcal{I}^1, j^1, \mathcal{I}^2, j^2, \cdots, \mathcal{I}^s, j^s \}$.

\vskip 0.2in
\noindent{\bf Case 2.} ${\nu(1-\eta \alpha) > 1}$. In this case,~\eqref{eq:chain} implies
\begin{align*}
{\E}_{\mathcal{I}^s|\xtilde} \Big[ \twonorm{\bx^m - \xhat}^2 \Big] \leq&\  \(\frac{2}{\alpha} + 4\nu \eta^2 Lm \)\left[ F(\xtilde) - F(\xhat) \right] + Q\\
&\ + \(\frac{2}{\alpha}\left[\nu(1 - \eta \alpha) - 1\right]m - 2\nu \eta (1 - 2\eta L)m \) {\E}_{\mathcal{I}^s, j^s|\xtilde} \left[ F(\xtilde^s) - F(\xhat) \right].
\end{align*}
Rearranging the terms gives
\begin{equation*}
\( 2\nu \eta \alpha - 2 \nu \eta^2 \alpha L - \nu +1 \)m\ {\E}_{\mathcal{I}^s, j^s|\xtilde} \left[ F(\xtilde^s) - F(\xhat) \right] \leq \(1 + 2\nu\eta^2 \alpha L m \)\left[ F(\xtilde) - F(\xhat) \right] + \frac{\alpha Q}{2}.
\end{equation*}
To ensure the convergence, the minimum requirements are
\begin{align*}
2\nu \eta \alpha - 2 \nu \eta^2 \alpha L - \nu +1 >&\ 0,\\
2\nu \eta \alpha - 2 \nu \eta^2 \alpha L - \nu +1 >&\ 2\nu\eta^2 \alpha L.
\end{align*}
That is,
\begin{equation*}
4\nu\alpha L \eta^2 - 2 \nu \alpha \eta + \nu - 1 < 0.
\end{equation*}
We need to guarantee the feasible set of the above inequality is non-empty for the positive variable $\eta$. Thus, we require
\begin{equation*}
4\nu^2 \alpha^2 - 4 \times 4\nu \alpha L (\nu - 1) > 0,
\end{equation*}
which is equivalent to
\begin{equation*}
\nu < \frac{4L}{4L - \alpha}.
\end{equation*}
Combining it with $\nu(1-\eta \alpha) > 1$ gives
\begin{equation*}
\frac{1}{ 1-\eta \alpha} < \nu < \frac{4L}{4L - \alpha}.
\end{equation*}
To ensure the above feasible set is non-empty, we impose
\begin{equation*}
\frac{1}{ 1-\eta \alpha} < \frac{4L}{4L - \alpha},
\end{equation*}
so that
\begin{equation}\label{eq:case2final}
0 < \eta < \frac{1}{4L},\quad \frac{1}{ 1-\eta \alpha} < \nu < \frac{4L}{4L - \alpha}.
\end{equation}
The convergence coefficient for this case is
\begin{equation}\label{eq:case2rate}
\beta = \frac{1}{\(2\nu \eta \alpha - 2 \nu \eta^2 \alpha L - \nu +1\)m} + \frac{2\nu\eta^2 \alpha L}{2\nu \eta \alpha - 2 \nu \eta^2 \alpha L - \nu +1}.
\end{equation}
Thus,
\begin{equation*}
\E \left[ F(\xtilde^s) - F(\xhat) \right] \leq \beta^s \left[ F(\xtilde^0) - F(\xhat) \right] + \frac{\alpha Q}{2(2\nu \eta \alpha - 2\nu \eta^2 \alpha L - \nu + 1)(1 - \beta)m}.
\end{equation*}

By combining \eqref{eq:case1final} and \eqref{eq:case2final}, the minimum requirement for $\eta$ and $\nu$ is
\begin{equation*}
0 < \eta < \frac{1}{4L},\quad \nu < \frac{4L}{4L - \alpha}.
\end{equation*}
The proof is complete.
\end{proof}

\subsection{Proof of Corollary~\ref{coro:param}}\label{sec:app:proofparam}
\begin{proof}
By noting the concavity of the square root function, we have
\begin{align*}
\E\Big[ \sqrt{ \max\{ F(\xtilde^s) - F(\xhat), 0 \} } \Big] &\leq \sqrt{ \E\big[ \max\{ F(\xtilde^s) - F(\xhat), 0 \} \big] }\\
&\leq \sqrt{ \( 2/3\)^s \max\{ F(\xtilde^0) - F(\xhat), 0 \} + \tau(\xhat) }.
\end{align*}
Suppose that $F(\bx)$ satisfies RSS with parameter $L' \in [\alpha, L]$. It follows that
\begin{equation*}
F(\xtilde^0) - F(\xhat) \leq \inner{ \nabla F(\xhat) }{\xtilde^0 - \xhat} + \frac{L'}{2} \twonorm{\xtilde^0 - \xhat}^2 \leq \frac{1}{2L'} \twonorm{\nabla_{k+K} F(\xhat)}^2 + L' \twonorm{\xtilde^0 - \xhat}^2.
\end{equation*}
Recall that
\begin{equation*}
\tau(\xhat) = \frac{5\omega}{\alpha} \twonorm{\nabla_{3k+K} F(\xhat)} + \frac{1}{\alpha L} \twonorm{\nabla_{3k+K} F(\xhat)}^2.
\end{equation*}
Hence using $\sqrt{a + b + c + d} \leq \sqrt{a} + \sqrt{b} + \sqrt{c} + \sqrt{d}$ gives
\begin{align*}
\E\Big[ \sqrt{ \max\{ F(\xtilde^s) - F(\xhat), 0 \} } \Big] \leq&\ \sqrt{L'} \( \frac{2}{3} \)^{\frac{s}{2}} \twonorm{\xtilde^0 - \xhat} + \sqrt{\frac{5\omega}{\alpha} \twonorm{\nabla_{3k+K} F(\xhat)} }\\
&\ + \( \frac{1}{\alpha} + \sqrt{\frac{1}{2\alpha}} \) \twonorm{\nabla_{3k+K} F(\xhat)}.
\end{align*}
Finally, the RSC property immediately suggests that (see, e.g., Lemma 20 in \cite{shen2017partial})
\begin{align*}
\E \big[\twonorm{\xtilde^s - \xhat} \big] \leq&\ \sqrt{\frac{2}{\alpha}} \E \Big[ \sqrt{ { \max\{ F(\xtilde^s) - F(\xhat), 0 \}} } \Big] + \frac{2 \twonorm{\nabla_{k+K} F(\xhat)}}{\alpha}\\
\leq&\ \sqrt{\frac{2L'}{\alpha}} \cdot \( \frac{2}{3} \)^{\frac{s}{2}} \twonorm{\xtilde^0 - \xhat} + \sqrt{\frac{10\omega}{\alpha^2} \twonorm{\nabla_{3k+K} F(\xhat)}} \\
&\ + \( \sqrt{\frac{2}{\alpha^3}}  + \frac{3}{\alpha} \) \twonorm{\nabla_{3k+K} F(\xhat)}.
\end{align*}
The proof is complete.
\end{proof}

\section{HT-SAGA}\label{sec:app:saga}

We demonstrate that the hard thresholding step can be integrated into SAGA~\cite{saga} as shown in Algorithm~\ref{alg:saga}. Note that the only difference of Algorithm~\ref{alg:saga} and the one proposed in~\cite{saga} is that we perform hard thresholding rather than proximal operator. Hence, our algorithm guarantees $k$-sparse solution. 

\begin{algorithm}[h]
	\caption{SAGA with Hard Thresholding~({\sc HT-SAGA})}
	\label{alg:saga}
	\begin{algorithmic}[1]
		\REQUIRE The current iterate $\bx^t$ and of each $\nabla f_i(\bphi_i^t)$ at the end of iteration $t$, the step size $\eta$.
		\ENSURE The new iterate.
		\STATE Pick $j \in \{1, 2, \cdots, n\}$ uniformly at random.
		\STATE Take $\bphi_j^{t+1} = \bx^t$ and store $\nabla f_j(\bphi_j^{t+1})$ in the table. All other entries in the table remain unchanged.
		\STATE Update the new iterate $\bx^{t+1}$ as follows:
		\begin{align*}
		\bb^{t+1} &= \bx^t - \eta \left[ \nabla f_j(\bphi_j^{t+1}) - \nabla f_j(\bphi_j^t) + \frac{1}{n}\sum_{i=1}^{n} \nabla f_i(\bphi_i^t) \right],\\
		\bx^{t+1} &= \Hk{\bb^{t+1}}.
		\end{align*}
	\end{algorithmic}
\end{algorithm}

\begin{theorem}
Assume the same conditions as in~\cite{saga}. Further assume the optimum of~\eqref{eq:primal} without the sparsity constraint happens to be $k$-sparse. Then, the sequence of the solutions produced by Algorithm~\ref{alg:saga} converges to the optimum with geometric rate for some properly chosen sparsity parameter $k$.
\end{theorem}

\begin{proof}
Define the Lyapunov function $Z$ as follows:
\begin{equation*}
Z^t \defeq Z(\bx^t, \{ \bphi_i^t \}) = \frac{1}{n} \sum_{i=1}^{n} f_i(\bphi_i^t) - F(\xhat) - \frac{1}{n} \sum_{i=1}^{n} \inner{\nabla f_i(\xhat)}{\bphi_i^t - \xhat} + c \twonorm{\bx^t - \xhat}^2.
\end{equation*}
We examine $Z^{t+1}$. We have
\begin{align*}
\E \left[ \frac{1}{n} \sum_i f_i(\bphi_i^{t+1}) \right] =&\ \frac{1}{n} F(\bx^t) + \(1-\frac{1}{n}\)\frac{1}{n} \sum_i f_i(\bphi_i^t),\\
\E \left[ -\frac{1}{n} \sum_i \inner{\nabla f_i(\xhat)}{\bphi_i^{t+1}-\xhat} \right] =&\ -\frac{1}{n}\inner{\nabla F(\xhat)}{\bx^t-\xhat} \\
&\ - \(1-\frac{1}{n}\)\frac{1}{n}\sum_i \inner{\nabla f_i(\xhat)}{\bphi_i^t - \xhat}.
\end{align*}
Also,
\begin{align*}
c\twonorm{\bx^{t+1} - \xhat}^2 &\leq c \nu \twonorm{{\bb^{t+1}}-\xhat}^2 = c \nu \twonorm{{\bb^{t+1}}-\xhat + \eta \nabla F(\xhat)}^2.
\end{align*}
For the first term, we have
\begin{align*}
&\ c \nu \E \twonorm{{\bb^{t+1}} - \xhat + \eta  {\nabla F(\xhat) }}^2 \notag\\
\leq&\ c\nu(1-\eta \alpha)\twonorm{\bx^t-\xhat}^2 + c\nu\( (1+\mu)\eta^2 - \frac{\eta}{L} \)\E \twonorm{\nabla f_j(\bx^t) - \nabla f_j(\xhat)}^2\notag\\
&\ - \frac{2c\nu \eta (L-\alpha)}{L} \left[ F(\bx^t) - F(\xhat) - \inner{\nabla F(\xhat)}{\bx^t - \xhat} \right] - c\nu \eta^2 \mu \twonorm{\nabla F(\bx^t) - \nabla F(\xhat)}^2\notag\\
&\ + 2c\nu (1+\mu^{-1}) \eta^2 L \left[ \frac{1}{n}\sum_i f_i(\bphi_i^t) - F(\xhat) - \frac{1}{n}\sum_i \inner{\nabla f_i(\xhat)}{\bphi_i^t - \xhat} \right].
\end{align*}
Therefore,
\begin{align*}
&\ \E[Z^{t+1}] - Z^t \\
\leq&\ -\frac{1}{\kappa} Z^t + \( \frac{1}{n} - \frac{2c\nu\eta(L-\alpha)}{L} - 2c\nu \eta^2\alpha \mu \)\left[ F(\bx^t) - F(\xhat) - \inner{\nabla F(\xhat)}{\bx^t - \xhat} \right] \notag\\
&\ + \( \frac{1}{\kappa} + 2c\nu(1+\mu^{-1})\eta^2L - \frac{1}{n} \)\left[ \frac{1}{n}\sum_i f_i(\bphi_i^t) - F(\xhat) - \frac{1}{n}\sum_i \inner{\nabla f_i(\xhat)}{\bphi_i^t - \xhat} \right] \notag\\
&\ + \(\frac{c}{\kappa} - c\nu\eta \alpha\) \twonorm{\bx^t - \xhat}^2 + \( (1+\mu)\eta - \frac{1}{L} \)c\nu\eta \E \twonorm{\nabla f_j(\bx^t) - \nabla f_j(\xhat)}^2.
\end{align*}
In order to guarantee the convergence, we choose proper values for $\eta$, $c$, $\kappa$, $\mu$ and $\nu$ such that the terms in round brackets are non-positive. That is, we require
\begin{align*}
\frac{c}{\kappa} - c\nu\eta \alpha &\leq 0,\\
(1+\mu)\eta - \frac{1}{L} &\leq 0,\\
\frac{1}{n} - \frac{2c\nu\eta(L-\alpha)}{L} - 2c\nu \eta^2\alpha \mu &\leq 0,\\
\frac{1}{\kappa} + 2c\nu(1+\mu^{-1})\eta^2L - \frac{1}{n} &\leq 0.
\end{align*}
Pick
\begin{align*}
\eta &= \frac{1}{2(\alpha n + L)},\\
\mu &= \frac{2\alpha n + L}{L},\\
\kappa &= \frac{1}{\nu \eta \alpha},
\end{align*}
we fulfill the first two inequalities. Pick
\begin{equation*}
c = \frac{1}{2\eta(1-\eta \alpha)n}.
\end{equation*}
Then by the last two equalities, we require
\begin{equation*}
1-\eta \alpha \leq \nu \leq \frac{(1-\eta \alpha)L}{\eta\alpha(1-\eta\alpha)Ln + 1}.
\end{equation*}
On the other hand, by Theorem~\ref{thm:key}, we have
\begin{equation*}
\nu > 1.
\end{equation*}
Thus, we require
\begin{equation*}
1 < \nu \leq \frac{(1-\eta \alpha)L}{\eta\alpha(1-\eta\alpha)Ln + 1},
\end{equation*}
By algebra, the above inequalities has non-empty feasible set provided that
\begin{equation*}
(6\alpha^2 - 8 \alpha^2L)n^2 + (14\alpha L - \alpha - 16\alpha L^2)n + 8L^2(1-L) < 0.
\end{equation*}
Due to $\alpha \leq L$, we know
\begin{equation*}
n \geq \frac{14L + \sqrt{224L^3+1}}{2\alpha(8L-6)}
\end{equation*}
suffices where we assume $L > 3/4$. Picking
\begin{equation*}
\nu = \frac{(1-\eta \alpha)L}{\eta\alpha(1-\eta\alpha)Ln + 1}
\end{equation*}
completes the proof.
\end{proof}

\bibliographystyle{alpha}
\bibliography{SVRG}
\end{document}